\documentclass{article} 
\usepackage[final]{colm2026_conference}

\usepackage{microtype}
\usepackage{hyperref}
\usepackage{url}
\usepackage{booktabs}

\usepackage{enumitem}
\usepackage{amsmath, amssymb, amsthm}
\usepackage{mathtools}
\usepackage{tikz}
\usepackage{graphicx}
\usepackage{booktabs}
\usetikzlibrary{arrows.meta, decorations.pathreplacing, calc, positioning}
 
\usepackage{lineno}

\newtheorem{theorem}{Theorem}[section]
\newtheorem{lemma}[theorem]{Lemma}

\newtheorem{corollary}[theorem]{Corollary}
\theoremstyle{definition}
\newtheorem{definition}[theorem]{Definition}
\newtheorem{assumption}[theorem]{Assumption}
\theoremstyle{remark}
\newtheorem{remark}[theorem]{Remark}

\newtheorem*{lemma*}{Lemma}
\newtheorem*{theorem*}{Theorem}
\newtheorem*{corollary*}{Corollary}
\newtheorem*{proposition*}{Proposition}
\newtheorem*{remark*}{Remark}

\newcommand{\R}{\mathbb{R}}
\newcommand{\E}{\mathbb{E}}
\newcommand{\Var}{\mathrm{Var}}
\newcommand{\tr}{\mathrm{tr}}
\newcommand{\diag}{\mathrm{diag}}
\newcommand{\xin}{x_{\mathrm{in}}}

\newcommand{\Dw}{\Delta W}

\definecolor{darkblue}{rgb}{0, 0, 0.5}
\hypersetup{colorlinks=true, citecolor=darkblue, linkcolor=darkblue, urlcolor=darkblue}

\usepackage{multirow}
\usepackage{makecell}
\usepackage{colortbl}
\usepackage{subcaption}
\usepackage{pgfplots}
\usetikzlibrary{pgfplots.groupplots}
\pgfplotsset{compat=1.18}
\usepackage{pifont}

\ifodd 1

\newcommand{\com}[1]{\textbf{\color{red}(comment: #1)}} 
\newcommand{\res}[1]{\textbf{\color{magenta}(RESPONSE: #1)}} 
\else

\newcommand{\com}[1]{}
\newcommand{\res}[1]{}
\fi

\newcommand{\spm}[1]{\textsc{#1}}
\newcommand{\scr}[1]{\textsc{#1}}
\definecolor{rowgray}{gray}{0.93}
\definecolor{goodgreen}{RGB}{0,140,0}
\definecolor{badred}{RGB}{200,0,0}

\title{When Pruning Meets Interpretability: Preserving Sparse Autoencoder Robustness in LLMs}

\author{Suchit Gupte, Xueru Zhang \& Mohammad Mahdi Khalili 
\\
Department of Computer Science\\
The Ohio State University\\
Columbus, OH 43210, USA \\
\texttt{\{gupte.31, zhang.12807, khalili.17\}@osu.edu} \\
}

\begin{document}

\ifcolmsubmission
\linenumbers
\fi

\maketitle

\begin{abstract}
Sparse autoencoders (SAEs) are widely used to interpret the internal representations of large language models (LLMs), yet their reliability under post-hoc model compression remains poorly understood. We present a systematic study of how pruning affects SAE behavior and theoretically show that, for a fixed SAE, its impact is governed by perturbation energy, a covariance-weighted norm. This perspective exposes a key limitation of \textsc{magnitude} pruning: by ignoring activation geometry, it distorts the learned representation space and degrades SAE functionality. Activation-aware methods such as \textsc{Wanda} and \textsc{SparseGPT}, in contrast, implicitly control perturbation energy and are therefore substantially more robust at preserving SAE behavior. We further reveal a consistent structural vulnerability across all pruning methods: middle layers are significantly more sensitive to pruning than early or late layers. Guided by this insight, we propose a Layer-wise sparsity allocation strategy, achieving lower perplexity under the same average pruning sparsity. Experiments across four model architectures validate our theoretical findings. The code is publicly available at \href{https://github.com/osu-srml/sae-robustness-under-pruning/tree/main}{sae-robustness-under-pruning}.

\end{abstract}

\section{Introduction} \label{sec:intro}

The superposition hypothesis~\citep{elhage2022toy} posits that neural networks represent more features than they have neurons, giving rise to polysemanticity where individual neurons respond to multiple unrelated concepts. Sparse autoencoders (SAEs)~\citep{sae} have emerged as the leading method for resolving this: by training an overcomplete sparse dictionary on a model's activations, they recover monosemantic, interpretable features~\citep{bricken2023monosemanticity,zhu2026abstopk}, an approach that scales to frontier models~\citep{templeton2024scaling, gemmascope}. SAE-based analysis has consequently become a core workflow in mechanistic interpretability~\citep{shu2025survey, sharkey2025open}. SAEs are used to identify circuits~\citep{conmy2023automated, wang2022ioi, absorption}, erase spurious correlations~\citep{scrtpp}, and test causal hypotheses about model behavior. All of these applications rely on the implicit assumption that any SAE provides a faithful and reliable decomposition of the underlying activations to which it is applied.

Independently, weight pruning has become a standard post-hoc technique for reducing the computational cost of deploying large language models. Methods ranging from classical \textsc{magnitude} pruning~\citep{magnitude} to modern one-shot approaches such as \textsc{SparseGPT}~\citep{sparsegpt} and \textsc{Wanda}~\citep{wanda} can remove a substantial fraction of model weights with only modest degradation in downstream task performance. Given the prohibitive cost of deploying large dense models, compression techniques including pruning, quantization, and distillation are standard tools in the open-weight deployment pipeline ~\citep{pruning1, pruning2, pruning3,cai2024privacyaware,zhu2025an}, often applied after interpretability analysis has been conducted on the dense checkpoint and any associated SAEs have been trained.  Crucially, the criteria by which these methods select weights to prune: whether by magnitude alone, by activation statistics, or by second-order Hessian, differ substantially in how they perturb the model's internal activation distribution, a distinction that has received little attention outside of predictive performance evaluations.



These two workflows are inherently in tension. An SAE is trained on the activation distribution of a specific dense model, whereas pruning alters that model's weight matrices, perturbing all downstream activations. When these perturbations remain small relative to the SAE's training distribution, the SAE may continue to function approximately as intended. Otherwise, its latents can fail silently~\citep{absorption, scrtpp, gupte2025transferabilitysparseautoencodersinterpreting}, and importantly, such failures may not be detected by standard language modeling benchmarks~\citep{egashira2025fewer, sharkey2025open}. Despite this, there is currently no theoretical framework that characterizes when pruning preserves SAE validity, nor an explanation for why different pruning methods lead to varying degrees of degradation.

A natural response to a changed model is to retrain the interpretability tooling itself. Crosscoders~\citep{lindsey}, for instance, train a joint dictionary across two model states (e.g., base and fine-tuned) to discover what changed between them. This is effective, but it requires training a new dictionary for each pair of model variants, precisely the retraining cost we aim to avoid. Our work instead asks: \textit{given an SAE already trained on a dense model, under what conditions does it remain valid after pruning, without retraining?} This is the more common deployment scenario, since organizations that have invested resources in SAE training need to know whether existing analysis transfers to a compressed model rather than needing to characterize what changed. The two approaches are thus complementary: crosscoders are the right tool for understanding a specific change, whereas answering the transfer question requires a theoretical account of when pruning leaves an existing SAE intact, which is exactly what is missing. This paper provides that account. 

We begin by theoretically examining the impact of \textsc{magnitude}, \textsc{SparseGPT}, and \textsc{Wanda} pruning, and show that SAE degradation is governed by \textbf{perturbation energy} ($\varepsilon^2$), a covariance-weighted norm that weights perturbations more heavily along high-variance input directions. This perspective establishes a formal separation among the three pruning methods: each method can be seen as minimizing a mathematically distinct approximation of $\varepsilon^2$. \textsc{Magnitude} pruning, which ignores activation geometry, induces the largest perturbation energy and the most severe SAE degradation, whereas activation-aware methods such as \textsc{Wanda} and \textsc{SparseGPT} implicitly control this quantity and are substantially more robust at preserving SAE behavior.

We validate our theory using four SAEBench~\citep{saebench} metric categories: Core~\citep{saebench}, Feature Absorption~\citep{absorption}, SCR, and TPP~\citep{scrtpp}, which together capture the full range of qualitative SAE failure modes. Beyond comparing methods, our experiments reveal a consistent structural vulnerability: \textbf{middle layers are significantly more sensitive to pruning than early or late layers}, across all methods and sparsity levels. Building on this insight, we explore a \textbf{layer-wise sparsity allocation strategy} that preserves SAE quality and achieves lower perplexity while maintaining the same average sparsity level.

The remainder of the paper is organized as follows. Sec.~\ref{sec:prelim} presents background on sparse autoencoders and formulates pruning as activation perturbation. Sec.~\ref{sec: theory} develops the perturbation-theoretic framework and derives the formal separation among pruning methods. Sec.~\ref{sec:saebench} argues for the sufficiency of the SAEBench evaluation suite in capturing all relevant SAE failure modes. Sec.~\ref{sec:setup} describes the experimental setup, and Sec.~\ref{sec:results} presents the results. Sec~\ref{sec:conclusion} concludes the paper and discusses limitations and future directions. 

\section{Preliminaries}
\label{sec:prelim}
 
\subsection{Sparse autoencoders for mechanistic interpretability} \label{subsec:saefmi}
 
\begin{definition}[Sparse Autoencoder]\label{def:sae}
A sparse autoencoder \citep{sae} is a map $f_\theta : \R^d \to \R^d$ with encoder $E_\theta : \R^d \to \R^m$ and decoder $D_\theta : \R^m \to \R^d$, where $m \gg d$:
\begin{equation}\label{eq:sae}
  f_\theta(x) = D_\theta\!\bigl(\sigma(E_\theta(x))\bigr), \quad
  E_\theta(x) = W_{\mathrm{enc}}\, x + b_{\mathrm{enc}}, \quad
  D_\theta(z) = W_{\mathrm{dec}}\, z + b_{\mathrm{dec}},
\end{equation}
where $\sigma$ is an arbitrary activation function, $W_{\mathrm{enc}} \in \R^{m \times d}$, $W_{\mathrm{dec}} \in \R^{d \times m}$.
\end{definition}
 
The SAE is trained on activations $x \sim \mathcal{D}$ from the model's forward pass to minimize:
\begin{equation}\label{eq:sae_loss}
  \mathcal{L}_{\mathrm{SAE}} = \E_{x \sim \mathcal{D}}\!\left[\|x - f_\theta(x)\|^2 + \lambda \|\sigma(E_\theta(x))\|_1\right].
\end{equation}
 
A key property we assume is that trained SAEs are Lipschitz continuous.
 
\begin{assumption}[Lipschitz continuity of the SAE] \label{asm:lip}
The trained SAE $f_\theta$ is $L$-Lipschitz: 
\begin{equation}
  \|f_\theta(x') - f_\theta(x)\| \leq L\, \|x' - x\|~ \forall x, x' \in \R^d.
\end{equation}
\end{assumption}
This assumption holds for standard SAE architectures whenever coordinate-wise activation function $\sigma$ is Lipschitz, as satisfied by common choices such as ReLU and GELU.
\subsection{Weight pruning as activation perturbation} \label{subsec:wpaap}
 
Consider a linear layer with weight matrix $W \in \R^{d \times p}$ producing activation $x = W\xin$, where $\xin \in \R^p$ is the layer input. Pruning modifies the weight matrix by zeroing a subset of entries $S \subset [d] \times [p]$, resulting in
$W' = W + \Dw$. The corresponding pruned activation is
\begin{equation}\label{eq:pert}
  x' = W'\xin = x + \delta, \qquad \delta \coloneqq \Dw\,\xin.
\end{equation}
We next summarize the objectives of common pruning methods.
\begin{definition}[Pruning objectives] \label{def:pruning}
Let $S$ be the pruning mask at target sparsity $s$.
\begin{enumerate}[noitemsep,topsep=0em,leftmargin=2em]
  \item \textbf{\textsc{magnitude} pruning} \citep{magnitude} selects $S$ to minimize $\|\Dw\|_F^2 = \sum_{(i,j)\in S} w_{ij}^2$.
  \item \textbf{\scr{Wanda}} \citep{wanda} selects $S$ to minimize $\sum_{(i,j)\in S} w_{ij}^2 \cdot \E \left[(\xin)^2_j\right]$.
  \item \textbf{\spm{SparseGPT}} \citep{sparsegpt} solves layer-wise regression to approximately minimize $\E \left[\|\Dw\,\xin\|^2\right]$ via second-order optimization.
\end{enumerate}
\end{definition}

\section{Perturbation Theory: SAE Degradation under Pruning} \label{sec: theory}

Next, we theoretically examine how pruning affects SAE behavior. Specifically, we develop a unified perturbation-theoretic analysis that bounds SAE degradation in terms of a data-dependent quantity, which we call the perturbation energy, together with the SAE's Lipschitz constant. This quantity provides a theoretical foundation for understanding the varying impact of different pruning techniques and explains why SAEs can be robust to some methods while vulnerable to others. Proof details are provided in Appendix~\ref{app:detailed_proofs}.

 

\textbf{Impact of pruning on SAE.} Recall from Sec.~\ref{subsec:wpaap} that pruning replaces $W$ with $W'=W+\Dw$, shifting the activation by $\delta = \Dw\,\xin$.  Because the SAE $f_\theta$ is $L$-Lipschitz (Assumption~\ref{asm:lip}), the triangle
inequality immediately gives a per-sample reconstruction bound.
\begin{lemma} \label{lem:recon}
Let $f_\theta$ be $L$-Lipschitz and let $x' = x + \delta$.  Then:
\begin{equation}\label{eq:lemma_restated}
  \|x' - f_\theta(x')\| \;\leq\; \|x - f_\theta(x)\| + (1+L)\|\delta\|.
\end{equation}
\end{lemma}

Squaring, taking expectations, and applying Cauchy--Schwarz to the cross term yields Thm.~\ref{thm:main}, which bounds expected degradation in SAE reconstruction quality after pruning.
 
\begin{theorem}[Expected reconstruction degradation]\label{thm:main}
Under Assumption~\ref{asm:lip},
\begin{equation}\label{eq:main}
  \E\bigl[\|x'-f_\theta(x')\|^2\bigr]
  \;\leq\;
  \underbrace{\E\bigl[\|x-f_\theta(x)\|^2\bigr]}_{\text{intrinsic SAE error}} + \,(1+L)^2\,\varepsilon^2 + 2(1+L)\sqrt{\E[\|x-f_\theta(x)\|^2]\cdot\varepsilon^2}
\end{equation}

where the \textbf{perturbation energy} $\varepsilon^2$ admits the covariance
decomposition
\begin{equation}\label{eq:eps}
  \varepsilon^2 \;:=\; \E \left[\|\delta\|^2\right]
  \;=\; \E \left[\|\Dw\,\xin\|^2\right]
  \;=\; \tr \bigl(\Dw\,\Sigma_{\xin}\,\Dw^\top\bigr),
\end{equation}
with $\Sigma_{\xin}=\E[\xin\xin^\top]$ the second-moment matrix of
layer inputs.
\end{theorem}

Although Thm.~\ref{thm:main}  is an upper bound, it provides insight into what drives SAE reconstruction degradation after pruning. The bound decomposes into two components: (1) \textit{intrinsic SAE error} which reflects how well the SAE was trained on the original model and is independent of pruning; (2) \textit{perturbation energy}  $\varepsilon^2$  which captures the direct damage caused by pruning. The third term captures their interaction, showing that a poorly trained SAE amplifies the effect of pruning: even small  $\varepsilon^2$ can cause disproportionately large degradation when the intrinsic error is large.

The bound also depends on the Lipschitz constant $L$, a property of the SAE architecture and training that is independent of the pruning method: both the $(1 + L)^2$ and cross terms grow with $L$, so a less smooth SAE amplifies any given perturbation energy. Since we hold SAE weights fixed throughout all experiments, $L$ is constant across pruning methods, making $\varepsilon^2$ the only operative variable when comparing methods; all method-level statements in this paper should be read under this fixed-SAE assumption. 

 
If we write the eigen decomposition $\Sigma_{\xin}=\sum_i\lambda_i v_i v_i^\top$
with $\lambda_1\geq\cdots\geq\lambda_p>0$, then
\begin{equation}\label{eq:eps_expand}
  \varepsilon^2 = \sum_{i}\lambda_i\,\|\Dw\,v_i\|^2.
\end{equation}
This is a \emph{covariance-weighted} squared norm of $\Dw$: each term $\lambda_i\,\|\Dw\,v_i\|^2$ quantifies how strongly $\Dw$ perturbs the
$i$-th principal direction of the input distribution, scaled by the variance along that direction. Perturbations aligned with high-variance directions contribute most to $\varepsilon^2$, while those along low-variance directions have little effect. 
The impact of pruning on SAEs therefore depends not only on how much the weights change, but on how those changes align with the input covariance structure, and pruning methods differ precisely in how well they account for this geometry.


 
 
\textbf{Separation among pruning methods.} The covariance weighting in \eqref{eq:eps_expand} induces a fundamental distinction between \textsc{magnitude} and activation-aware pruning methods.
 
\begin{corollary}[Comparison of pruning methods]\label{cor:paradigm}
At any target sparsity $s$:
\begin{itemize}[noitemsep,topsep=0em,leftmargin=2em]
\item \textbf{\textsc{magnitude} pruning} minimizes $\|\Dw\|_F^2=\tr(\Dw\,I\,\Dw^\top)$, which coincides with minimizing $\varepsilon^2$ \emph{only} when $\Sigma_{\xin}\propto I$.
  \item \textbf{\scr{Wanda}} minimizes
  $\tr(\Dw\, \diag(\E[\xin\xin^\top])\, \Dw^\top)$,
        a diagonal approximation to $\varepsilon^2$.
  \item \textbf{\spm{SparseGPT}} approximately minimizes $\varepsilon^2$ directly via
        second-order information.
\end{itemize}
\end{corollary}
 Corollary~\ref{cor:paradigm} highlights that the three pruning methods differ fundamentally in how they approximate $\varepsilon^2$: \textsc{Magnitude} pruning, by ignoring the covariance structure of the activations, approximates $\varepsilon^2$ poorly and consequently induces larger perturbation energy and more severe SAE degradation. In contrast, \textsc{Wanda} and \textsc{SparseGPT} implicitly account for this structure, better controlling $\varepsilon^2$ and thus more effectively preserving SAE behavior.

This separation can be made precise: the three methods form a natural hierarchy in how faithfully they approximate the true objective $\varepsilon^2$:
\begin{equation}
  \underbrace{\tr(\Dw\, I\, \Dw^\top)}_{\text{\textsc{magnitude}}}
  \;\prec\;
  \underbrace{\tr(\Dw\, \diag(\Sigma_{\xin})\, \Dw^\top)}_{\text{\scr{Wanda}}}
  \;\prec\;
  \underbrace{\tr(\Dw\, \Sigma_{\xin}\, \Dw^\top)}_{\text{SparseGPT}},
\end{equation}
where $\prec$ denotes ``is a coarser approximation of the true
perturbation energy than''. \textsc{magnitude} pruning ignores $\Sigma_{\xin}$ entirely; \scr{Wanda} uses only its diagonal; SparseGPT uses the full matrix, mirroring exactly the ordering of SAE degradation we observe empirically in Sec \ref{sec:results}.

\section{From Theory to Evaluation: SAEBench as a Sufficient Validation Suite}\label{sec:saebench}
Sec.~\ref{sec: theory} shows that, for a fixed SAE, $\varepsilon^2$ governs SAE degradation under pruning. Validating this empirically requires metrics that collectively cover every way an SAE can fail when its input distribution shifts. We argue that the four SAEBench categories, \emph{Core}, \emph{Feature Absorption}, \emph{SCR}, and \emph{TPP}, are jointly sufficient as they cover three orthogonal failure modes. Omitting any one would leave a blind spot in the validation. We describe each failure mode and its corresponding metric in turn, then argue their joint sufficiency (see Appendix~\ref{app:metrics} for metric definitions). 



\textbf{Failure Mode 1: Reconstruction Degrades.} When pruning shifts the activation distribution, the SAE may no longer accurately reconstruct its inputs. Since the SAE's reconstruction is inserted back into the model's forward pass, degraded reconstruction directly corrupts the model's predictions. The \underline{\textbf{Core}} metrics \citep{saebench} detect this failure by measuring how well the SAE's reconstruction preserves model behavior, using the KL divergence score, CE loss score, explained variance, MSE, and cosine similarity (see Appendix~\ref{app:metrics:core}). Preserved Core metrics indicate that the SAE remains within its effective reconstruction region after pruning. However, Core metrics are aggregated over all latents and tokens, making them insensitive to localized failures: a score of 0.99 is consistent with 5\% of latents behaving incorrectly, provided the rest compensate. The SAE can silently reorganize its internal representations while keeping aggregate scores high, leaving the failure modes below entirely undetected.

\textbf{Failure Mode 2: Features Break.} 
Even when aggregate reconstruction appears intact, individual latents may fail. A latent that should activate on a given input may go silent (false negative), or one that should be silent may fire spuriously (false positive). When this happens, the SAE no longer provides a faithful, interpretable decomposition of the activations, even though nothing appears wrong at the aggregate level. \underline{\textbf{Feature Absorption}} \citep{absorption} detects this failure by testing whether the SAE maintains clean feature separation after pruning. Using a first-letter classification task, it identifies cases where an expected latent fails to activate and checks whether a semantically unrelated latent has absorbed its information. A lower absorption rate indicates better feature separation. Thus, the absorption metric goes beneath the aggregate surface and checks whether a feature that should represent ``starts with S'' still fires on S-tokens, or whether its information has been absorbed by a semantically unrelated latent. Without this metric, we could observe perfect Core scores while the SAE has been silently rearranged into an uninterpretable state.

\textbf{Failure Mode 3: Interventions Stop Working.}
The primary use of SAEs in interpretability is causal intervention: ablating specific latents to test whether they causally mediate a concept or behavior. Pruning can break this connection: a latent may still activate on the right inputs but no longer causally influence the model's output as expected, or its ablation may collaterally affect unrelated concepts.  \underline{\textbf{Spurious Correlation Removal (SCR)}} \citep{scrtpp} detects this failure by testing whether latent ablations still produce their intended causal effect after pruning: if ablating latents $S_c$ removed a spurious ``$\texttt{gender}\to\texttt{profession}$" correlation in the dense model, SCR checks whether the same ablation achieves the same effect after pruning. Without SCR, we could observe latents that activate correctly yet have lost their causal influence: the latent fires on the right inputs, but ablating it no longer changes the model's behavior. \underline{\textbf{Targeted Probe Perturbation (TPP)}} \citep{scrtpp} complements SCR by testing \textit{selectivity}: ablating latents for concept $c$ should degrade performance on class $c$ but leave other classes unaffected. Without TPP, we could conclude that interventions remain effective while missing that they have become non-specific, damaging concepts they were never intended to affect.

\textbf{Joint Sufficiency.}
Together, these metrics form a logical chain that mirrors the structure of our perturbation theory. \textbf{Core} verifies that $\varepsilon^2$ is small enough for aggregate reconstruction to be preserved (Thm.~\ref{thm:main}). \textbf{Feature Absorption} verifies that $\varepsilon^2$ is small enough relative to activation margins, ensuring individual latents remain intact. \textbf{SCR} verifies that surviving features retain their causal role, so that latent-level interventions still produce the intended effect. \textbf{TPP} verifies that these causal effects remain selective, that concepts stay disentangled, not merely detectable. This progression from aggregate fidelity, to feature-level correctness, to interventional validity, to interventional selectivity covers the full stack of properties that make SAEs useful as an interpretability tool. If all four are preserved under pruning, every standard SAE-based analysis transfers from the dense model to the pruned model.

\section{Experimental Setup}
\label{sec:setup}
 
\textbf{Models and SAEs.}
To isolate the effect of pruning from SAE training variability, we hold all SAE weights fixed and use pretrained checkpoints from SAELens~\citep{saelens}. Our model selection is governed by the availability of high-quality pretrained SAEs.  By fixing pretrained SAEs, we ensure that any observed degradation is attributable solely to pruning and not to SAE training choices. This consideration leads us to four model--SAE pairs spanning roughly two orders of magnitude in parameter count: \texttt{pythia-70m}~\citep{pythia, pythiasaes}, \texttt{gemma-2-2b}~\citep{gemma2, gemmascope}, \texttt{gemma-2-9b}~\citep{gemma2, gemmascope}, and \texttt{mistral-7b}~\citep{mistral, mistralsaes}, all of which have publicly available, well-validated SAE checkpoints. Our main experiments focus on \texttt{gemma-2-2b}, which benefits from the most comprehensive SAE coverage among the four models. This model has pretrained SAEs available at every residual-stream layer (layers 0--25), enabling the fine-grained layer-level analysis that is central to our investigation.  Results for all three remaining models are reported in Sec.~\ref{sec:results} as a cross-model generalisation study, where we verify that the theoretical ordering and layer sensitivity patterns observed on \texttt{gemma-2-2b} hold across architectures and scales.


\textbf{Pruning methods.} We apply three pruning methods post-hoc to each dense checkpoint: \textsc{magnitude} pruning~\citep{magnitude}, \textsc{Wanda}~\citep{wanda}, and \textsc{SparseGpt}~\citep{sparsegpt} (see Sec.~\ref{subsec:wpaap} for an overview). No retraining or weight updates are performed after pruning. Calibration data for \textsc{Wanda} and \textsc{SparseGpt} consists of 128 samples from OpenWebText~\citep{Gokaslan2019OpenWeb}. We conduct experiments at 50\% sparsity, a practically relevant operating point at which all three methods remain functional and differences in SAE degradation are clearly visible.  To characterise how SAE degradation scales with the degree of compression, we additionally evaluate at 25\% and 40\% sparsity, allowing us to assess whether the layer sensitivity patterns and method ordering identified at 50\% are consistent across sparsity levels or specific to that operating point.

\textbf{Evaluation.} For each pruned model variant, we compute the four SAEBench~\citep{saebench} metric categories described in Sec.~\ref{sec:saebench}: Core (KL divergence score, CE loss score, explained variance, cosine similarity, MSE), Feature Absorption, SCR, and TPP. Evaluations are run on the residual output of every layer (layers 0--25 for \texttt{gemma-2-2b}) to expose layer-level sensitivity patterns. For SCR and TPP, we adopt $k = 10$ as the canonical ablation threshold, as it lies in the stable regime of the threshold sensitivity curve and produces the clearest separation between pruning methods (see Appendix~\ref{app:metrics}). All results are reported both as raw metric values and as percentage change from the dense baseline.

\section{Results} \label{sec:results}

\begin{table*}[t]
\centering
\setlength{\tabcolsep}{3pt}
\small
\begin{tabular}{@{}llr ccccc cc c c@{}}
\toprule
\multirow{2}{*}{\textbf{Layers}} &\multirow{2}{*}{\textbf{Method}} &\multirow{2}{*}{\makecell{\textbf{Spar.}\\\textbf{(\%)}}} &\multicolumn{5}{c}{\textbf{Core}} &\multicolumn{2}{c}{\textbf{Absorption}} &\textbf{SCR} &\textbf{TPP} \\
\cmidrule(lr){4-8}\cmidrule(lr){9-10}\cmidrule(lr){11-11}\cmidrule(lr){12-12}
& & &\makecell{KL\\Score$\uparrow$} &\makecell{CE\\Score$\uparrow$} &\makecell{Expl.\\Var.$\uparrow$} &\makecell{Cos\\Sim$\uparrow$} &\makecell{MSE\\$\downarrow$} &\makecell{Abs.\\Frac.$\downarrow$} &\makecell{Full\\Abs.$\downarrow$} &\makecell{Top-10\\$\uparrow$} &\makecell{Top-10\\$\uparrow$} \\
\midrule
\multirow{4}{*}{\rotatebox[origin=c]{90}{\texttt{\shortstack{Early \\ (0-8)}}}}
 & \textit{Baseline} & \textit{0}
   & 0.9944 & 0.9941 & 0.8976 & 0.9497 & 0.4268
   & 0.0674 & 0.0354
   & 0.1714 & 0.0270 \\
\cmidrule(l){2-12}
 & \spm{Magnitude} & 50
   & 0.9847 & 0.9818 & 0.8416 & 0.9175 & 0.6536
   & 0.0764 & 0.0486
   & 0.1130 & 0.0187 \\
 & \spm{Wanda} & 50
   & \textbf{0.9935} & 0.9873 & \textbf{0.8793} & \textbf{0.9401} & \textbf{0.4648}
   & \textbf{0.0691} & \textbf{0.0353}
   & \textbf{0.1487} & 0.0188 \\
 & \spm{SparseGPT} & 50
   & 0.9918 & \textbf{0.9906} & 0.8728 & 0.9353 & 0.4886
   & 0.0717 & 0.0402
   & 0.1329 & \textbf{0.0196} \\
\midrule
\multirow{4}{*}{\rotatebox[origin=c]{90}{\texttt{\shortstack{Middle \\ (9-17)}}}}
 & \textit{Baseline} & \textit{0}
   & 0.9887 & 0.9876 & 0.8763 & 0.9301 & 2.0582
   & 0.0668 & 0.0586
   & 0.3047 & 0.0424 \\
\cmidrule(l){2-12}
 & \spm{Magnitude} & 50
   & 0.9644 & 0.9531 & 0.7782 & 0.8780 & 3.3021
   & 0.0794 & 0.0730
   & 0.1301 & 0.0089 \\
 & \spm{Wanda} & 50
   & \textbf{0.9894} & 0.9854 & \textbf{0.8507} & \textbf{0.9175} & \textbf{2.2214}
   & \textbf{0.0660} & \textbf{0.0564}
   & \textbf{0.2399} & \textbf{0.0304} \\
 & \spm{SparseGPT} & 50
   & 0.9868 & \textbf{0.9893} & 0.8277 & 0.9071 & 2.5616
   & 0.0731 & 0.0662
   & 0.2103 & 0.0258 \\
\midrule
\multirow{4}{*}{\rotatebox[origin=c]{90}{\texttt{\shortstack{Late \\ (18-25)}}}}
 & \textit{Baseline} & \textit{0}
   & 0.9776 & 0.9765 & 0.8682 & 0.9307 & 14.2891
   & 0.1208 & 0.1369
   & 0.3885 & 0.0681 \\
\cmidrule(l){2-12}
 & \spm{Magnitude} & 50
   & 0.9535 & 0.9510 & 0.7544 & 0.8701 & 29.4766
   & \textbf{0.0831} & \textbf{0.1177}
   & 0.3192 & 0.0117 \\
 & \spm{Wanda} & 50
   & \textbf{0.9780} & 0.9773 & \textbf{0.8569} & \textbf{0.9253} & \textbf{14.7422}
   & 0.1045 & 0.1179
   & \textbf{0.3743} & 0.0519 \\
 & \spm{SparseGPT} & 50
   & 0.9768 & \textbf{0.9815} & 0.8389 & 0.9150 & 17.6016
   & 0.1054 & 0.1293
   & 0.3605 & \textbf{0.0547} \\
\bottomrule
\end{tabular}
\caption{SAEBench metrics at 50\% sparsity for \texttt{gemma-2-2b}, aggregated by layer group (Early: layers 0--8; Middle: 9--17; Late: 18--25). Each group reports the unpruned dense baseline (italicised) alongside all three pruning methods. $\uparrow$ higher is better; $\downarrow$ lower is better. Refer to App.~\ref{app:perlayer-gemma2b} for per-layer SAEBench metrics from which these values are aggregated. }
\label{tab:results_by_group}
\end{table*}

\begin{figure}[t]
    \centering
    \includegraphics[width=\linewidth]{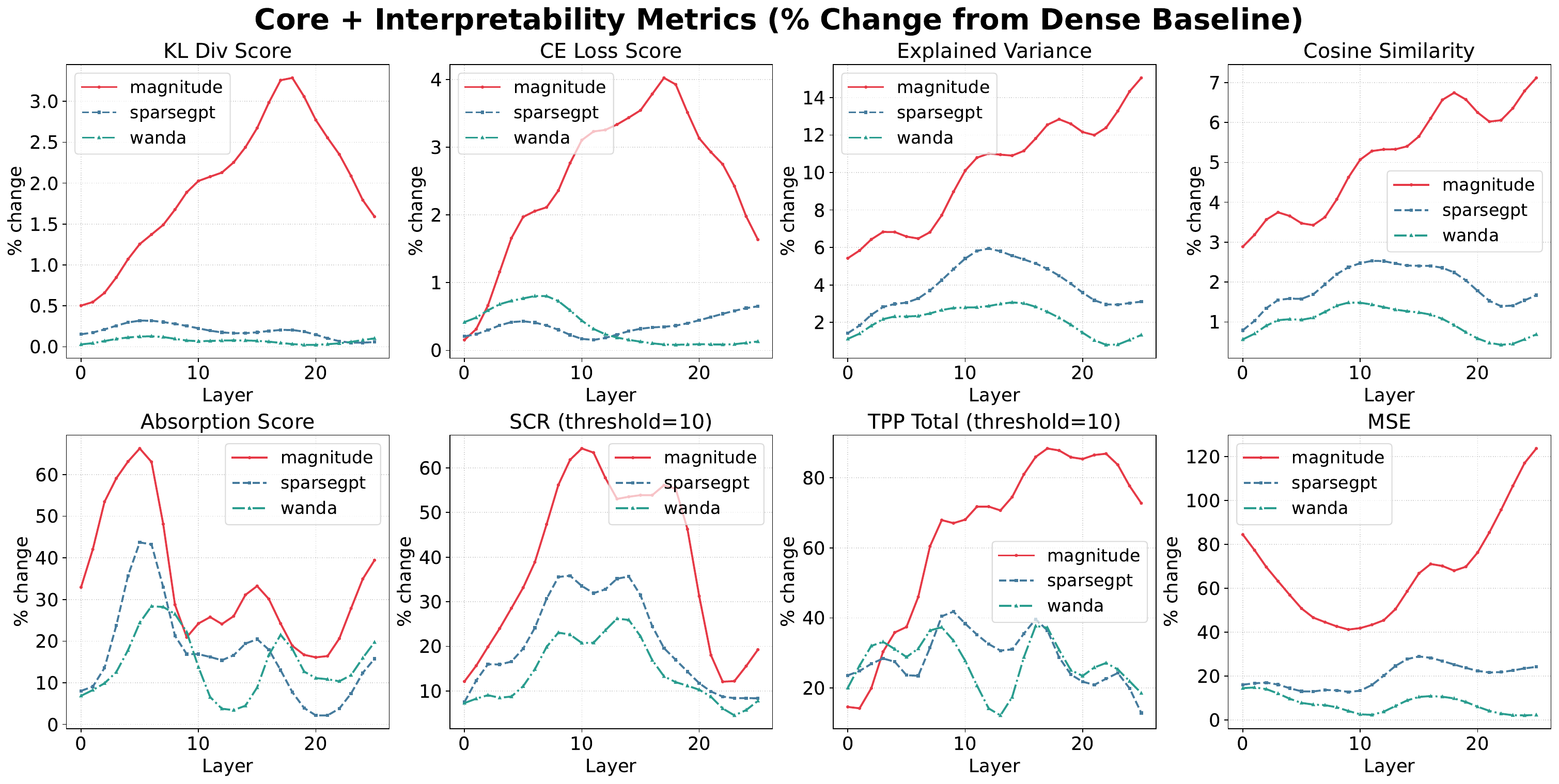}
    \caption{Per-layer percentage degradation from the dense baseline across all eight SAEBench metric categories for \texttt{gemma-2-2b} at 50\% sparsity. Each panel corresponds to one metric; \textsc{Magnitude} (red), \textsc{SparseGpt} (blue), and \textsc{Wanda} (green) across layers 0--25. These plots emphasize per-layer trends and direct cross-method comparison at each layer; see Figure \ref{fig:heatmap} for the same data as a heatmap. }
    \vspace{-1em}
    \label{fig:pct_change}
\end{figure}

\textbf{Impact of pruning methods is consistent and theory-aligned.}
Table~\ref{tab:results_by_group} and Figure~\ref{fig:pct_change} together reveal a clear and consistent ordering across all SAEBench metric categories: \textsc{magnitude} pruning causes the most severe SAE degradation, \textsc{Wanda} is substantially more robust, and \textsc{SparseGpt} is equally conservative, closely tracking the dense baseline. This ordering holds without exception across all eight metrics and all 26 layers of \texttt{gemma-2-2b} at 50\% sparsity.

The pattern precisely validates our perturbation-theoretic analysis: \textsc{magnitude} pruning minimises the wrong objective: the Frobenius norm, which equals perturbation energy only under an identity covariance. While \textsc{Wanda} and \textsc{SparseGpt} progressively better approximate the true perturbation energy $\varepsilon^2$. The separation is particularly stark on interventional metrics: \textsc{magnitude} pruning degrades SCR by up to 60\% and TPP by up to 80\% relative to the dense baseline in the worst-affected layers, whereas \textsc{Wanda} and \textsc{SparseGpt} remain within 15--25\% of baseline on the same metrics (Refer to Figure~\ref{fig:pct_change}).

Critically, the Core metrics, which aggregate over all latents, substantially understate the damage inflicted by \textsc{magnitude} pruning. For instance, KL divergence scores for \textsc{magnitude} pruning in middle layers appear superficially reasonable (${\geq}0.96$), yet the corresponding SCR and TPP scores collapse, confirming our theoretical prediction that aggregate reconstruction fidelity is insufficient to certify SAE reliability. Appendix~\ref{app:pert-energy} closes the loop directly: measured $\varepsilon^2$ values confirm the ordering of Corollary~\ref{cor:paradigm} at every layer group, and reveal that middle layers have the \emph{lowest} local $\varepsilon^2$ despite the highest degradation, implicating accumulated upstream perturbation.

\begin{figure}[t]
    \centering
    \includegraphics[width=\linewidth]{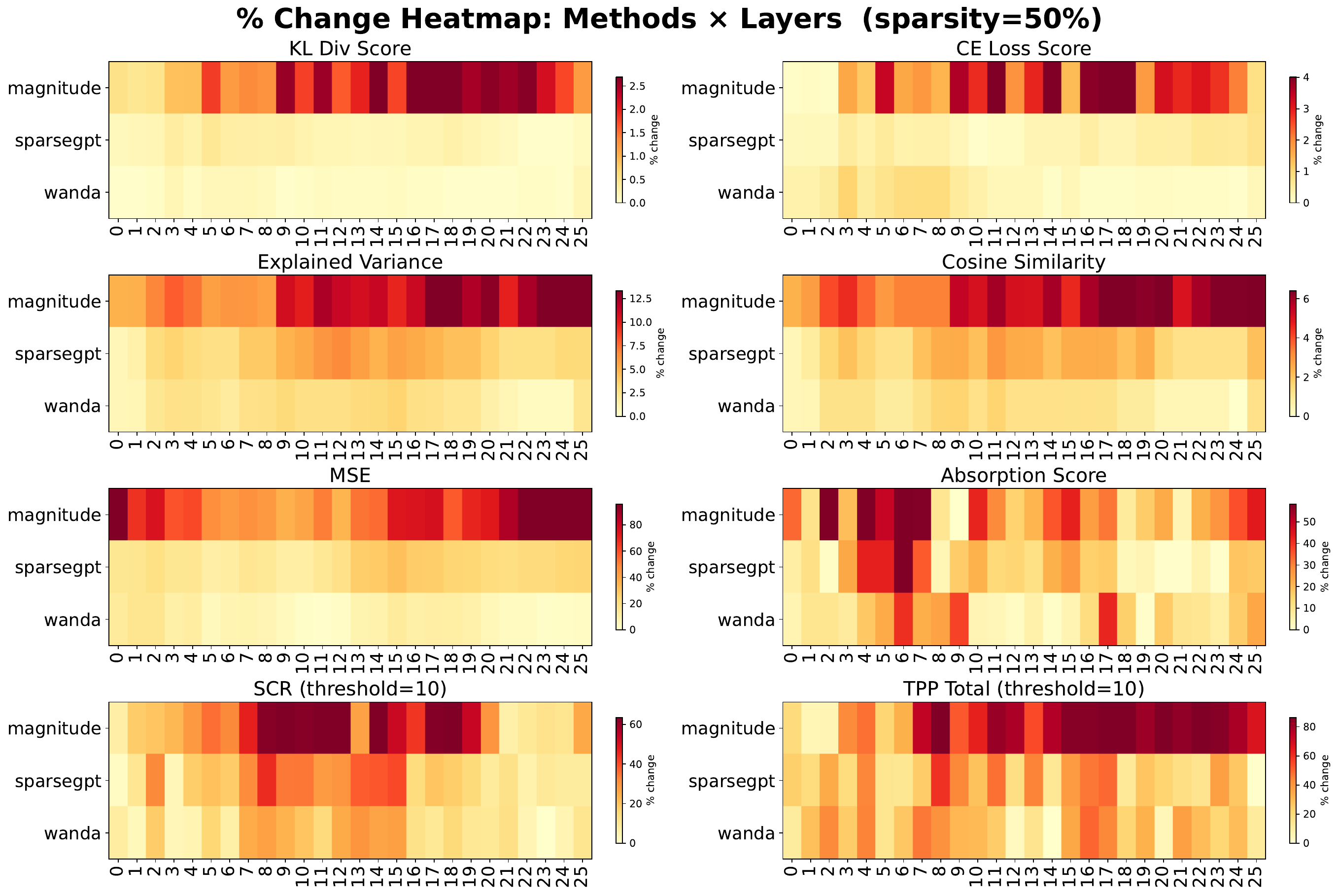}
    \caption{Heatmap of per-layer, per-method percentage degradation from the dense baseline for \texttt{gemma-2-2b} at 50\% sparsity. Each row corresponds to a pruning method; each column to a layer; colour intensity encodes the magnitude of degradation (darker $=$ larger change). This heatmap presents the same underlying data as Figure \ref{fig:pct_change}, reorganized to expose spatial patterns across the full method $\times$ layer grid simultaneously.}
    \label{fig:heatmap}
\end{figure}

\begin{figure}[t]
    \centering
    \includegraphics[width=\linewidth]{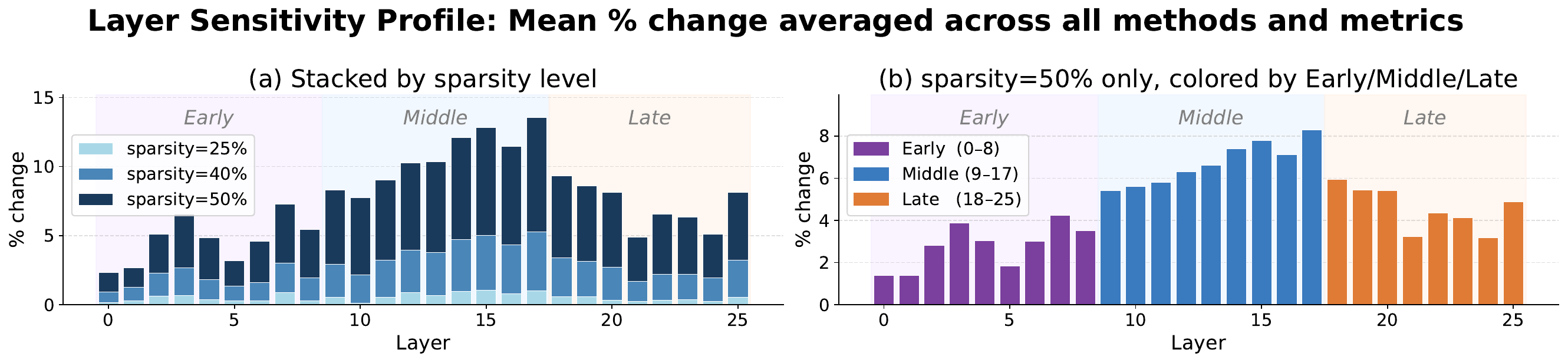}
    \caption{Layer sensitivity profile for \texttt{gemma-2-2b}, computed as the mean absolute percentage degradation averaged across all three pruning methods and all eight SAEBench metrics. \textbf{(a)} Sensitivity profiles at three sparsity levels (25\%, 40\%, 50\%). \textbf{(b)} Sensitivity profile at 50\% sparsity, coloured by layer group (Early: 0--8, Middle: 9--17, Late: 18--25).}
    \vspace{-0.5em}
    \label{fig:layer_sensitivity}
\end{figure}

\begin{table*}[ht]
\centering
\setlength{\tabcolsep}{3pt}
\small
\begin{tabular}{@{}llr ccccc cc c c@{}}
\toprule
\multirow{2}{*}{\textbf{Model}} &\multirow{2}{*}{\textbf{Method}} &\multirow{2}{*}{\makecell{\textbf{Spar.}\\\textbf{(\%)}}} &\multicolumn{5}{c}{\textbf{Core}} &\multicolumn{2}{c}{\textbf{Absorption}} &\textbf{SCR} &\textbf{TPP} \\
\cmidrule(lr){4-8}\cmidrule(lr){9-10}\cmidrule(lr){11-11}\cmidrule(lr){12-12}
& & &\makecell{KL\\Score$\uparrow$} &\makecell{CE\\Score$\uparrow$} &\makecell{Expl.\\Var.$\uparrow$} &\makecell{Cos\\Sim$\uparrow$} &\makecell{MSE\\$\downarrow$} &\makecell{Abs.\\Frac.$\downarrow$} &\makecell{Full\\Abs.$\downarrow$} &\makecell{Top-10\\$\uparrow$} &\makecell{Top-10\\$\uparrow$} \\
\midrule
\multirow{4}{*}{\rotatebox[origin=c]{90}{\texttt{\shortstack{gemma-2 \\ 2b}}}}
 & \textit{Baseline} & \textit{0}
   & 0.9873 & 0.9864 & 0.8812 & 0.9370 & 5.2568
   & 0.0836 & 0.0746
   & 0.2843 & 0.0450 \\
\cmidrule(l){2-12}
 & \spm{Magnitude} & 50
   & 0.9681 & 0.9624 & 0.7928 & 0.8893 & 10.4390
   & 0.0795 & 0.0783
   & 0.1824 & 0.0131 \\
 & \spm{Wanda} & 50
   & \textbf{0.9873} & 0.9835 & \textbf{0.8625} & \textbf{0.9277} & \textbf{5.4659}
   & \textbf{0.0789} & \textbf{0.0680}
   & \textbf{0.2497} & \textbf{0.0330} \\
 & \spm{SparseGPT} & 50
   & 0.9855 & \textbf{0.9874} & 0.8468 & 0.9193 & 6.4717
   & 0.0826 & 0.0766
   & 0.2297 & 0.0325 \\
\midrule
\multirow{4}{*}{\rotatebox[origin=c]{90}{\texttt{\shortstack{gemma-2 \\ 9b}}}}
 & \textit{Baseline} & \textit{0}
   & 0.9914 & 0.9900 & 0.8882 & 0.9380 & 7.7362
   & 0.0997 & 0.0654
   & 0.3237 & 0.0603 \\
\cmidrule(l){2-12}
 & \spm{Magnitude} & 50
   & 0.9826 & 0.9757 & 0.8506 & 0.9173 & 12.4004
   & 0.0919 & 0.0619
   & 0.2307 & 0.0298 \\
 & \spm{Wanda} & 50
   & \textbf{0.9911} & 0.9869 & \textbf{0.8772} & \textbf{0.9322} & \textbf{7.7391}
   & \textbf{0.0905} & 0.0608
   & \textbf{0.2921} & \textbf{0.0409} \\
 & \spm{SparseGPT} & 50
   & 0.9901 & \textbf{0.9892} & 0.8627 & 0.9245 & 9.6890
   & 0.0910 & \textbf{0.0606}
   & 0.2706 & 0.0395 \\
\midrule
\multirow{4}{*}{\rotatebox[origin=c]{90}{\texttt{\shortstack{pythia \\ 70m}}}}
 & \textit{Baseline} & \textit{0}
   & 0.9262 & 0.9240 & 0.9476 & 0.9572 & 0.0118
   & N/A$^\dagger$ & N/A$^\dagger$
   & 0.0726 & 0.0505 \\
\cmidrule(l){2-12}
 & \spm{Magnitude} & 50
   & 0.6422 & 0.8740 & 0.5592 & 0.8765 & 0.0615
   & N/A$^\dagger$ & N/A$^\dagger$
   & 0.0054 & \textbf{0.0420} \\
 & \spm{Wanda} & 50
   & 0.8798 & 0.9091 & 0.8705 & 0.9244 & \textbf{0.0195}
   & N/A$^\dagger$ & N/A$^\dagger$
   & 0.0150 & 0.0405 \\
 & \spm{SparseGPT} & 50
   & \textbf{0.8819} & \textbf{0.9198} & \textbf{0.9025} & \textbf{0.9328} & 0.0204
   & N/A$^\dagger$ & N/A$^\dagger$
   & \textbf{0.0321} & 0.0409 \\
\midrule
\multirow{4}{*}{\rotatebox[origin=c]{90}{\texttt{\shortstack{mistral \\ 7b}}}}
 & \textit{Baseline} & \textit{0}
   & 0.9722 & 0.9753 & 0.9844 & 0.8490 & 0.0150
   & 0.4292 & 0.4150
   & 0.3627 & 0.0501 \\
\cmidrule(l){2-12}
 & \spm{Magnitude} & 50
   & 0.9672 & \textbf{0.9758} & 0.9792 & 0.8034 & 0.0177
   & 0.4510 & 0.4371
   & 0.3634 & 0.0208 \\
 & \spm{Wanda} & 50
   & 0.9742 & 0.9720 & 0.9857 & \textbf{0.8477} & \textbf{0.0128}
   & 0.4023 & 0.3805
   & \textbf{0.3810} & \textbf{0.0463} \\
 & \spm{SparseGPT} & 50
   & \textbf{0.9749} & 0.9734 & \textbf{0.9870} & \textbf{0.8477} & 0.0131
   & \textbf{0.3848} & \textbf{0.3715}
   & 0.3617 & 0.0396 \\
\bottomrule
\end{tabular}
\vspace{6pt}
\begin{minipage}{\linewidth}
  \scriptsize
  $^\dagger$ Feature absorption evaluation requires sufficient first-letter features to be detectable in the model. SAEBench~\citep{saebench} issues a warning that this evaluation is only reliable for models $\geq$2B parameters; pythia-70m~\citep{pythia} (70M parameters) does not produce enough first-letter features for a valid absorption measurement~\citep{absorption, pythiasaes}, so these cells are excluded from all absorption analyses and averages. \\[2pt] $^\ddagger$ Pretrained SAEs for mistral-7b are available for only one layer per depth group (early, middle, late) in SAELens~\citep{saelens,mistralsaes}; mistral-7b values are therefore averaged over 3 layers. The method ordering is consistent with the full-coverage models, but cross-model magnitude comparisons involving mistral-7b should be interpreted with this coverage difference in mind.
\end{minipage}
\vspace{-1em}
\caption{Cross-architecture generalisation at 50\% sparsity. Each cell reports the metric value averaged across all evaluated layers of a given model. $\uparrow$ higher is better; $\downarrow$ lower is better.}
\label{tab:results_global_avg}
\vspace{-1.5em}
\end{table*}

\textbf{Middle layers exhibit systematic structural vulnerability.}
A striking and consistent finding across all methods and metrics is the elevated sensitivity of middle layers (9--17) to pruning, relative to both early layers (0--8) and late layers (18--25). Figure~\ref{fig:layer_sensitivity} makes this explicit: averaging the percentage degradation across all methods and metrics, the middle-layer profile rises sharply above that of early and late layers at every sparsity level tested (25\%, 40\%, 50\%), and this pattern is monotone in sparsity, confirming it is not an artifact of a particular operating point. At 50\% sparsity, the heatmap in Figure~\ref{fig:heatmap} shows that middle layers carry the darkest cells across nearly every metric--method combination, with \textsc{Magnitude} pruning on middle layers producing degradations of 10--14\% in explained variance, 5--7\% in cosine similarity, and 40--60\% in SCR and TPP, substantially exceeding the corresponding values for early and late layers.

We hypothesize this reflects the asymmetric propagation of representational damage across depth. Early layers construct foundational features that are heavily inherited by middle layers via residual connections, making middle layers the primary recipients of perturbation energy $\varepsilon^2$ introduced by pruning upstream weights. Late layers, by contrast, operate on an already fully composed representation and specialize in output-facing transformations, leaving them with little residual upstream dependency to corrupt. This asymmetry directly motivates a layer-wise sparsity schedule.


\textbf{Layer-wise sparsity allocation.}
The layer sensitivity profile identified earlier suggests that a uniform sparsity schedule is suboptimal. Motivated by this hypothesis, we experiment with redistributing the global compression budget unevenly across depth: \textit{assigning lower sparsity to early layers to preserve the activations that middle layers depend on, and higher sparsity to late layers where upstream perturbation has a negligible downstream consequence}. Concretely, given a target global sparsity $\bar{s}$, we assign layer-wise sparsity $s_\ell$ according to a schedule that ramps linearly from $s_{\min}$ at the first layer to a plateau $s_{\text{mid}}$ at the middle-layer boundary, and holds at $s_{\text{mid}}$ through the remaining layers. 



As a preliminary evaluation, we compare the perplexity of the layer-wise schedule against that of uniform sparsity at equivalent average sparsity on \texttt{gemma-2-2b}, using both \textsc{Wanda} and \textsc{SparseGpt}. Perplexity serves as a coarse measure to evaluate the model as a whole. As seen in Table \ref{tab:dynamic_pruning}, the layer-wise schedule yields lower perplexity in both cases, suggesting that the layer sensitivity profile is a practically useful signal for compression scheduling. We view this as an encouraging preliminary result rather than a fully developed method. 

Non-uniform layerwise sparsity has previously been explored for general LLM compression. OWL~\citep{yin2025outlierweighedlayerwisesparsity} allocates sparsity in proportion to each layer's activation-outlier ratio, validated against perplexity and downstream task accuracy. Our layer-wise schedule is motivated by a different, interpretability-specific diagnostic. The two criteria are not guaranteed to coincide: outlier concentration and SAE-degradation sensitivity are distinct properties of a layer's activation distribution, though a systematic comparison is orthogonal to this paper's central contribution, and we leave it to future work.

\begin{table}[t]
    \centering 
    \setlength{\tabcolsep}{3pt}
    \footnotesize
    \begin{tabular}{lccccc}
        \toprule
        & & \multicolumn{2}{c}{\textbf{Wanda}} & \multicolumn{2}{c}{\textbf{SparseGPT}} \\
        \cmidrule(lr){3-4} \cmidrule(lr){5-6}
        \textbf{Model} & \textbf{Avg. Sparsity}  & Uniform &  Layer-wise  & Uniform &  Layer-wise \\
        \midrule
        \texttt{gemma-2-2b} & 50\% & 212 & \textbf{148} & 141 & \textbf{98} \\
        \bottomrule
    \end{tabular}
     \caption{Perplexity on \texttt{gemma-2-2b} under uniform versus layer-wise sparsity allocation evaluated for \textsc{Wanda} and \textsc{SparseGpt}. The ramp schedule assigns 20\% sparsity to early layers, increasing linearly to 60\% at the middle-layer and holding at 60\%  until the last layer.}\label{tab:dynamic_pruning}
    \vspace{-1em}
\end{table}

\textbf{Cross-model generalisation confirms the theoretical ordering.}
Table~\ref{tab:results_global_avg} reports results aggregated over all layers for four model architectures spanning roughly two orders of magnitude in parameter count: \texttt{pythia-70m}, \texttt{gemma-2-2b}, \texttt{gemma-2-9b}, and \texttt{mistral-7b}. Across all models for which valid SAE evaluations are available, the method ordering $\textsc{SparseGpt} \succ \textsc{Wanda} \succ \textsc{Magnitude}$ is preserved on Core metrics, with \textsc{Wanda} and \textsc{SparseGpt} achieving near-baseline KL and CE scores while \textsc{magnitude} pruning causes marked degradation. The result is especially pronounced for \texttt{pythia-70m}, where \textsc{magnitude} pruning reduces the KL divergence score from $0.926$ to $0.642$, a collapse not observed under \textsc{Wanda} ($0.880$) or \textsc{SparseGpt} ($0.882$). \texttt{Mistral-7b} also displays comparatively mild absolute degradation across all three methods on Core metrics.

\section{Conclusion} \label{sec:conclusion}
This paper studies the interaction between weight pruning and mechanistic interpretability via sparse autoencoders (SAE), showing that, for a fixed SAE, pruning-induced SAE degradation is governed by perturbation energy, a covariance-weighted norm that measures how strongly pruning distorts the activation distribution, modulated by the SAE's Lipschitz constant. This framework explains the consistent empirical ordering of pruning methods: \textsc{Magnitude} pruning, which ignores activation geometry, induces large perturbation energy and significantly degrades SAE behavior, while activation-aware methods such as \textsc{Wanda} and \textsc{SparseGpt} better preserve interpretability by controlling perturbation energy. We validate this theory through a comprehensive evaluation across \texttt{gemma-2-2b}, \texttt{gemma-2-9b}, \texttt{mistral-7b}, and \texttt{pythia-70m}, using pretrained SAEs and the SAEBench evaluation suite. Additionally, we identify a structural vulnerability in middle layers and show that layer-wise sparsity allocation, protecting these layers better, preserves representations under equivalent compression.

\textbf{Limitations and Future Work.}  \label{sec:limitations}
Our study has several limitations that point toward promising directions for future work.
\textbf{First}, we hold SAE weights fixed after pruning. An important complementary direction is to retrain SAEs from scratch on the pruned model's activations and examine whether the recovered features remain consistent with those of the dense model. \textbf{Second}, our conclusions rest on four model--SAE pairs drawn from a relatively homogeneous set of architectures. It remains unclear whether the results generalize to other model families, instruction-tuned models, mixture-of-experts architectures, or SAEs trained on MLP and attention-head outputs rather than the residual stream. \textbf{Third}, our analysis is confined to unstructured pruning; developing analogous theory for structured pruning, quantization, and knowledge distillation is important as these methods grow more prevalent in production deployments. \textbf{Fourth}, our layer-wise sparsity allocation is a preliminary illustration that the sensitivity profile is a useful scheduling signal, not a proposed allocation method in its own right: we do not benchmark it against more non-intuitive methods or dedicated layerwise-allocation methods such as OWL~\cite{yin2025outlierweighedlayerwisesparsity}, since a rigorous comparison is a separate undertaking from the SAE-degradation analysis that is this paper's focus. \textbf{Finally}, while the middle-layer sensitivity finding is empirically robust, our framework does not yet explain it theoretically.  A formal account of how the structure of internal representations evolves with depth, and how this motivates principled rather than heuristic sparsity allocation, remains an open problem.

\section*{Acknowledgments} This work was funded in part by the National Science Foundation under award number IIS2202699, IIS-2416895, IIS-2301599, CMMI2301601, DMS-2529302.
We used large language models (Claude by Anthropic, ChatGPT by OpenAI) to assist with writing and paraphrasing portions of this paper. No LLMs were used to originate research ideas, generate data or plots, or perform evaluation.

\bibliography{colm2026_conference}
\bibliographystyle{colm2026_conference}

\appendix
\section{Appendix} \label{app}
 
\subsection{Detailed Proofs for the Perturbation-Theoretic Framework}
\label{app:detailed_proofs}
 
This appendix provides fully detailed, self-contained proofs of the
perturbation theory presented in Section \ref{sec: theory}.
 
\subsubsection{Setup and Notation}
 
We recall the objects involved.
 
\begin{itemize}
  \item $W \in \R^{d \times p}$: the weight matrix of a single linear
        layer.
  \item $\xin \in \R^p$: the random input to that layer, drawn from
        the data distribution $\mathcal{D}$.
  \item $x = W\xin \in \R^d$: the original (pre-pruning) activation.
  \item $\Dw \in \R^{d \times p}$: the weight \emph{change} induced by
        pruning. Because pruning zeroes entries, every non-zero entry
        $(i,j)$ of $\Dw$ satisfies $(\Dw)_{ij} = -W_{ij}$.
  \item $x' = (W+\Dw)\xin = x + \delta$: the pruned activation, where
        $\delta := \Dw\,\xin$ is the \textbf{perturbation vector}.
  \item $f_\theta : \R^d \to \R^d$: the trained SAE, which is
        $L$-Lipschitz (Assumption~\ref{asm:lip}).
  \item $\Sigma_{\xin} := \E[\xin \xin^\top] \in \R^{p \times p}$: the
        second-moment (input covariance) matrix.
\end{itemize}
 
We want to bound the \textbf{expected squared
reconstruction error on pruned activations}:
\begin{equation}
  \mathcal{E}_{\mathrm{pruned}}
  \;:=\;
  \E \left[\|x' - f_\theta(x')\|^2\right].
\end{equation}
We will show that this is bounded in terms of the
\textbf{perturbation energy}
\begin{equation}
  \varepsilon^2 \;:=\; \E \left[\|\delta\|^2\right],
\end{equation}
and then show that $\varepsilon^2 = \tr(\Dw\,\Sigma_{\xin}\,\Dw^\top)$.
 
\subsubsection{Proof of Lemma~\ref{lem:recon}}
 
\begin{lemma*}[Lemma~\ref{lem:recon}, restated]
Let $f_\theta$ be $L$-Lipschitz and let $x' = x + \delta$.  Then:
\begin{equation}\label{eq:lemma_restated}
  \|x' - f_\theta(x')\| \;\leq\; \|x - f_\theta(x)\| + (1+L)\|\delta\|.
\end{equation}
\end{lemma*}
 
\begin{proof}
We want to bound the reconstruction error at $x'$.
We will relate it to the reconstruction error at $x$, which we
know is small, and the perturbation $\delta$, which we want to bound.


\textbf{Applying the triangle inequality}: The triangle inequality states that for any three vectors $a, b, c$:
$\|a - c\| \leq \|a - b\| + \|b - c\|$.
We use it with $a = x'$, $c = f_\theta(x')$, and $b$ chosen to introduce $f_\theta(x)$:
\begin{align}
  \|x' - f_\theta(x')\|
  &= \|(x' - f_\theta(x)) + (f_\theta(x) - f_\theta(x'))\|
     \notag \\
  &\leq \|x' - f_\theta(x)\| + \|f_\theta(x) - f_\theta(x')\|.
     \label{eq:tri1}
\end{align}
We now bound each of these two terms separately.
 
 
\textbf{Bounding the first term}: Since $x' = x + \delta$, we apply the triangle inequality again:
\begin{align}
  \|x' - f_\theta(x)\|
  &= \|(x + \delta) - f_\theta(x)\| \notag \\
  &= \|(x - f_\theta(x)) + \delta\| \notag \\
  &\leq \|x - f_\theta(x)\| + \|\delta\|. \label{eq:tri2}
\end{align}
 
 
\textbf{Bounding the second term}: We use the Lipschitz continuity assumption here. By Assumption~\ref{asm:lip}, the SAE satisfies
$\|f_\theta(u) - f_\theta(v)\| \leq L\|u - v\|$ for all $u, v$.
Applying this with $u = x$ and $v = x' = x + \delta$:
\begin{equation}\label{eq:lip_applied}
  \|f_\theta(x) - f_\theta(x')\|
  \;\leq\; L\|x - x'\|
  \;=\; L\|(x) - (x+\delta)\|
  \;=\; L\|\delta\|.
\end{equation}
 
Substituting \eqref{eq:tri2} and \eqref{eq:lip_applied} into
\eqref{eq:tri1}:
\begin{align}
  \|x' - f_\theta(x')\|
  &\leq \|x' - f_\theta(x)\| + \|f_\theta(x) - f_\theta(x')\|
     \notag \\
  &\leq \bigl(\|x - f_\theta(x)\| + \|\delta\|\bigr) + L\|\delta\|
     \notag \\
  &= \|x - f_\theta(x)\| + (1 + L)\|\delta\|. \qedhere
\end{align}
\end{proof}
 
\begin{remark*}[Interpretation of $(1+L)$]
The factor $(1+L)$ has a precise meaning.
The $+1$ accounts for the direct shift of the activation: moving from
$x$ to $x' = x+\delta$ shifts the reconstruction target by $\|\delta\|$.
The $+L$ accounts for the SAE's \emph{response} to that shift: because
the SAE is Lipschitz with constant $L$, its output can move by at most
$L\|\delta\|$ when the input moves by $\|\delta\|$.
Both movements contribute additively to the reconstruction error.
\end{remark*}
 
\textbf{Proof of Theorem~\ref{thm:main} (Expected Bound and
Covariance Decomposition)}
 
\begin{theorem*}[Theorem~\ref{thm:main}, restated]
Under Assumption~\ref{asm:lip}:
\begin{equation}\label{eq:thm_restated}
  \E \left[\|x'-f_\theta(x')\|^2\right]
  \;\leq\;
  \E \left[\|x-f_\theta(x)\|^2\right]
  + (1+L)^2\,\varepsilon^2
  + 2(1+L)\sqrt{\E[\|x-f_\theta(x)\|^2]\cdot\varepsilon^2},
\end{equation}
where
\begin{equation}\label{eq:covariance_id}
  \varepsilon^2
  \;:=\; \E \left[\|\delta\|^2\right]
  \;=\; \tr \bigl(\Dw\,\Sigma_{\xin}\,\Dw^\top\bigr).
\end{equation}
\end{theorem*}
 
\begin{proof}
We prove the two parts separately.
 
\bigskip
\noindent\textbf{Part 1: Proving the bound \eqref{eq:thm_restated}.}
 
From Lemma~\ref{lem:recon}, we have for every sample:
\begin{equation}
  \|x' - f_\theta(x')\|
  \;\leq\; \|x - f_\theta(x)\| + (1+L)\|\delta\|.
\end{equation}
Let us write $A := \|x - f_\theta(x)\|$ and $B := (1+L)\|\delta\|$
to keep the algebra clean.
The bound says $\|x' - f_\theta(x')\| \leq A + B$, so since squaring
is monotone on non-negative numbers:
\begin{equation}
  \|x' - f_\theta(x')\|^2
  \;\leq\; (A + B)^2
  \;=\; A^2 + 2AB + B^2.
\end{equation}
Expanding back:
\begin{equation}\label{eq:squared_bound}
  \|x' - f_\theta(x')\|^2
  \;\leq\;
  \|x - f_\theta(x)\|^2
  + 2(1+L)\|x - f_\theta(x)\|\cdot\|\delta\|
  + (1+L)^2\|\delta\|^2.
\end{equation}
 
 
$\E[\,\cdot\,]$ is linear, so taking expectations on both sides of
\eqref{eq:squared_bound}:
\begin{align}
  \E \left[\|x' - f_\theta(x')\|^2\right]
  &\leq
  \E \left[\|x - f_\theta(x)\|^2\right]
  + (1+L)^2\E \left[\|\delta\|^2\right] \notag \\
  &\quad + 2(1+L)\,\E \left[\|x - f_\theta(x)\|\cdot\|\delta\|\right].
  \label{eq:after_expectation}
\end{align}
 
 
The cross term $\E[\|x - f_\theta(x)\|\cdot\|\delta\|]$ involves the
expectation of a \emph{product} of two random variables.
We bound it using the Cauchy--Schwarz inequality for expectations,
which states that for any two random variables $U, V$:
\begin{equation}
  \E[UV] \;\leq\; \sqrt{\E[U^2]\,\E[V^2]}.
\end{equation}
Setting $U = \|x - f_\theta(x)\|$ and $V = \|\delta\|$:
\begin{equation}\label{eq:cs}
  \E \left[\|x - f_\theta(x)\|\cdot\|\delta\|\right]
  \;\leq\;
  \sqrt{\E \left[\|x - f_\theta(x)\|^2\right] \cdot \E \left[\|\delta\|^2\right]}.
\end{equation}
 
 
Substituting \eqref{eq:cs} into \eqref{eq:after_expectation} and
writing $\varepsilon^2 := \E[\|\delta\|^2]$:
\begin{equation}
  \E \left[\|x' - f_\theta(x')\|^2\right]
  \;\leq\;
  \E \left[\|x - f_\theta(x)\|^2\right]
  + (1+L)^2\,\varepsilon^2
  + 2(1+L)\sqrt{\E \left[\|x - f_\theta(x)\|^2\right]\cdot\varepsilon^2}.
\end{equation}
This is precisely \eqref{eq:thm_restated}. \hfill$\square_{\text{Part 1}}$
 
\bigskip
\noindent\textbf{Part 2: Proving the covariance identity
\eqref{eq:covariance_id}.}
 
We show that $\E[\|\delta\|^2] = \tr(\Dw\,\Sigma_{\xin}\,\Dw^\top)$.
 
 
Recall $\delta = \Dw\,\xin$.
The squared norm of a vector $v$ can be written as the dot product
$\|v\|^2 = v^\top v$.
Therefore:
\begin{equation}
  \|\delta\|^2
  \;=\; \delta^\top\delta
  \;=\; (\Dw\,\xin)^\top(\Dw\,\xin)
  \;=\; \xin^\top \Dw^\top \Dw\,\xin.
\end{equation}
 
 
For any deterministic matrix $M \in \R^{p \times p}$ and any random
vector $v \in \R^p$, following is a fundamental identity:
\begin{equation}\label{eq:trace_trick}
  \E \left[v^\top M v\right] \;=\; \tr \bigl(M\,\E[vv^\top]\bigr).
\end{equation}
\emph{Why does this identity hold?}
Writing out the quadratic form entry-by-entry:
$v^\top M v = \sum_{i,j} M_{ij} v_i v_j$.
Taking expectations and using linearity:
$\E[v^\top M v] = \sum_{i,j} M_{ij}\,\E[v_i v_j]$.
But $\E[v_i v_j] = [\E[vv^\top]]_{ij}$, so:
\begin{equation}
  \E[v^\top M v]
  \;=\; \sum_{i,j} M_{ij}\,[\E[vv^\top]]_{ij}
  \;=\; \sum_i [M\,\E[vv^\top]]_{ii}
  \;=\; \tr  \bigl(M\,\E[vv^\top]\bigr),
\end{equation}
where the second equality uses the definition of matrix
multiplication, and the third uses the definition of the trace of a matrix.
 
 
Applying \eqref{eq:trace_trick} with $M = \Dw^\top\Dw$ and
$v = \xin$:
\begin{equation}
  \E \left[\|\delta\|^2\right]
  \;=\; \E \left[\xin^\top \Dw^\top\Dw\,\xin\right]
  \;=\; \tr\bigl(\Dw^\top\Dw\,\E[\xin\xin^\top]\bigr)
  \;=\; \tr \bigl(\Dw^\top\Dw\,\Sigma_{\xin}\bigr).
\end{equation}
 
 
The trace satisfies $\tr(AB) = \tr(BA)$ for any matrices $A, B$ of
compatible dimensions (according to the \emph{cyclic property}).
Applying this with $A = \Dw^\top$ and $B = \Dw\,\Sigma_{\xin}$:
\begin{equation}
  \tr \bigl(\Dw^\top\Dw\,\Sigma_{\xin}\bigr)
  \;=\; \tr \bigl(\Dw\,\Sigma_{\xin}\,\Dw^\top\bigr).
\end{equation}
 
 
Chaining the equalities :
\begin{equation}
  \varepsilon^2
  \;=\; \E \left[\|\delta\|^2\right]
  \;=\; \E \left[\|\Dw\,\xin\|^2\right]
  \;=\; \tr \bigl(\Dw\,\Sigma_{\xin}\,\Dw^\top\bigr). \qedhere
\end{equation}
\end{proof}
 
 
\begin{remark*}[Eigenvalue expansion of $\varepsilon^2$]
\label{rem:eigenexpand}
Writing the eigendecomposition $\Sigma_{\xin} = \sum_{i=1}^p
\lambda_i v_i v_i^\top$ with $\lambda_1 \geq \cdots \geq \lambda_p
> 0$, the cyclic property gives:
\begin{equation}\label{eq:eig_expand}
  \varepsilon^2
  \;=\; \tr(\Dw\,\Sigma_{\xin}\,\Dw^\top)
  \;=\; \sum_{i=1}^p \lambda_i \,\|\Dw\, v_i\|^2.
\end{equation}
Each term $\lambda_i \|\Dw\, v_i\|^2$ is the product of
(i) the variance of the data in direction $v_i$, and
(ii) the squared magnitude of the activation perturbation in that
direction.
Equation~\eqref{eq:eig_expand} makes the covariance-weighting
explicit: a pruning decision that perturbs a high-variance direction
(large $\lambda_i$) is penalised proportionally more in $\varepsilon^2$,
and therefore causes proportionally more damage to the SAE.
\end{remark*}

\subsubsection{Proof of Corollary~\ref{cor:paradigm} (Paradigm
Comparison)}
 
\begin{corollary*}[Corollary~\ref{cor:paradigm}, restated]
At any target sparsity $s$, the three pruning methods relate to
$\varepsilon^2$ as follows:
\begin{enumerate}
  \item \textsc{magnitude} pruning minimises
        $\|\Dw\|_F^2 = \tr(\Dw\, I\, \Dw^\top)$.
  \item \scr{Wanda} minimises
        $\tr\!\bigl(\Dw\,\diag(\E[\xin\xin^\top])\,\Dw^\top\bigr)$.
  \item SparseGPT approximately minimises $\varepsilon^2 =
        \tr(\Dw\,\Sigma_{\xin}\,\Dw^\top)$ directly.
\end{enumerate}
\end{corollary*}
 
\begin{proof}
Each claim follows directly from the definition of the method
(Definition~\ref{def:pruning}) combined with the identity
$\varepsilon^2 = \tr(\Dw\,\Sigma_{\xin}\,\Dw^\top)$ from
Theorem~\ref{thm:main}.
 
\medskip
\noindent\textbf{Claim 1 (\textsc{magnitude} pruning).}
 
\textsc{magnitude} pruning selects the pruning mask to minimise
$\sum_{(i,j)\in S} w_{ij}^2 = \|\Dw\|_F^2$ (since $(\Dw)_{ij} = -w_{ij}$
for pruned entries).
By the covariance identity, $\|\Dw\|_F^2 = \tr(\Dw\,I\,\Dw^\top)$,
which equals $\tr(\Dw\,\Sigma_{\xin}\,\Dw^\top)$ only when
$\Sigma_{\xin} = I$.
When $\Sigma_{\xin} \neq I$ (the typical case), \textsc{magnitude} pruning
minimises the wrong objective.
 
\medskip
\noindent\textbf{Claim 2 (\scr{Wanda}).}
 
\scr{Wanda} scores each entry $(i,j)$ by $w_{ij}^2 \cdot \E[(x_{\mathrm{in}})_j^2]$
and prunes the lowest-scoring entries.
The summed score of all pruned entries is
$\sum_{(i,j)\in S} w_{ij}^2 \cdot \E[(x_{\mathrm{in}})_j^2]
= \tr\!\bigl(\Dw\,D\,\Dw^\top\bigr)$,
where $D = \diag(\E[(x_{\mathrm{in}})^2])$ is the diagonal matrix of
per-coordinate input variances.
Since $D_{jj} = \E[(x_{\mathrm{in}})_j^2]$, and noting that
$[\Sigma_{\xin}]_{jj} = \E[(x_{\mathrm{in}})_j^2]$,
the \scr{Wanda} objective is the \emph{diagonal approximation} to
$\tr(\Dw\,\Sigma_{\xin}\,\Dw^\top)$, obtained by discarding the
off-diagonal (cross-covariance) terms of $\Sigma_{\xin}$.
 
\medskip
\noindent\textbf{Claim 3 (SparseGPT).}
 
SparseGPT solves the layer-wise reconstruction problem
$\min_{\Dw : \mathrm{sparsity}(W+\Dw)=s} \E[\|\Dw\,\xin\|^2]$
via a second-order approximation using the empirical Hessian
$\hat{\Sigma}_{\xin} = \frac{1}{N}\sum_{n=1}^N \xin^{(n)}(\xin^{(n)})^\top$.
By Theorem~\ref{thm:main}, $\E[\|\Dw\,\xin\|^2] = \tr(\Dw\,\Sigma_{\xin}\,\Dw^\top)$,
so SparseGPT directly (approximately) minimises $\varepsilon^2$.
\end{proof}
 
\begin{remark}[Hierarchy of approximations]
The three methods form a natural hierarchy in how faithfully they
approximate the true objective $\varepsilon^2$:
\begin{equation}
  \underbrace{\tr(\Dw\, I\, \Dw^\top)}_{\text{\textsc{magnitude}}}
  \;\prec\;
  \underbrace{\tr(\Dw\, \diag(\Sigma_{\xin})\, \Dw^\top)}_{\text{\scr{Wanda}}}
  \;\prec\;
  \underbrace{\tr(\Dw\, \Sigma_{\xin}\, \Dw^\top)}_{\text{SparseGPT}},
\end{equation}
where $\prec$ denotes ``is a coarser approximation of the true
perturbation energy than''.
\textsc{magnitude} pruning ignores $\Sigma_{\xin}$ entirely; \scr{Wanda} uses only its
diagonal; SparseGPT uses the full matrix.
\end{remark}

\subsection{SAEBench Evaluation Suite} \label{app:metrics}
This appendix gives a detailed overview of the SAEBench evaluation suite~\citep{saebench}. We also provide intuitive understanding of each metric category~\citep{absorption, scrtpp, saebench}. Ultimately, we conclude the appendix by talking about the rationale behind the SCR and TPP threshold choices.

\subsubsection{Core metrics} \label{app:metrics:core}
\paragraph{KL divergence score.}
Let $\mathcal{D}_{\text{orig}}$ denote the output distribution of the original model 
and $\mathcal{D}_{\text{SAE}}$ the output distribution when the residual stream is 
replaced by the SAE reconstruction. The zero-ablation distribution 
$\mathcal{D}_{\text{zero}}$ is obtained by replacing the residual stream with zeros. 
The KL divergence score is defined as:
\begin{equation}
    \text{KL score} = 1 - 
    \frac{D_{\mathrm{KL}}(\mathcal{D}_{\text{orig}} \,\|\, \mathcal{D}_{\text{SAE}})}
         {D_{\mathrm{KL}}(\mathcal{D}_{\text{orig}} \,\|\, \mathcal{D}_{\text{zero}})},
\end{equation}
normalized so that a perfect reconstruction yields a score of 1 and a zero-ablation 
yields a score of 0. This metric is sensitive to distributional shifts in the model's 
predictions, not just pointwise reconstruction error.

\paragraph{CE loss score.}
Let $H_{\text{orig}}$, $H^{*}$, and $H_{\text{ablation}}$ denote the cross-entropy 
loss of the model under the original activations, SAE-reconstructed activations, and 
zero-ablated activations, respectively. The CE loss score is defined as:
\begin{equation}
    \text{CE score} = 
    \frac{H_{\text{ablation}} - H^{*}}{H_{\text{ablation}} - H_{\text{orig}}}.
\end{equation}
A score of 1 indicates that the SAE reconstruction perfectly recovers the original 
model's cross-entropy loss, while a score of 0 corresponds to the zero-ablation 
baseline. Unlike the KL divergence score, this metric emphasizes the impact of 
reconstruction quality on next-token prediction performance.

\paragraph{Explained variance.}
For a layer with activation $x \in \mathbb{R}^d$ and SAE reconstruction 
$\hat{x} = f_\theta(x)$, the explained variance is:
\begin{equation}
    \text{EV} = 1 - \frac{\mathrm{Var}(x - \hat{x})}{\mathrm{Var}(x)},
\end{equation}
where variance is computed over the token dimension. A higher explained variance 
indicates that the SAE accounts for more of the variance in the original activations. 
This metric is sensitive to systematic reconstruction failures that affect large 
portions of the activation space.

\paragraph{Mean squared error (MSE).}
The MSE between the original activation $x$ and its reconstruction $\hat{x}$ is:
\begin{equation}
    \text{MSE} = \mathbb{E}\left[\|x - f_\theta(x)\|^2\right],
\end{equation}
measured in the same units as the activation space. Unlike the normalized metrics above, 
MSE is an absolute measure and therefore varies substantially across layers and model 
sizes, as later layers tend to have activations with larger norms. We report MSE 
alongside the normalized metrics to provide a scale-sensitive view of reconstruction 
quality.

\paragraph{Cosine similarity.}
The mean cosine similarity between original and reconstructed activations is:
\begin{equation}
    \text{Cos Sim} = \mathbb{E}\left[
        \frac{\langle x,\, f_\theta(x) \rangle}{\|x\|\,\|f_\theta(x)\|}
    \right].
\end{equation}
This metric captures directional alignment between original and reconstructed 
activations, independent of their magnitudes. It is particularly informative when 
activation norms vary substantially across layers, since magnitude-preserving but 
direction-distorting perturbations appear benign under MSE yet are detected here.

\subsubsection{Feature Absorption} \label{app:metrics:absorption}
Feature absorption happens when SAEs, due to sparsity pressure, fail to represent a general feature directly. Instead of having both a \texttt{general} latent (e.g., “starts with S”) and \texttt{specific} latents (e.g., “short”), the SAE only activates the \texttt{specific} ones. This makes the \texttt{general} latent look unreliable $\implies$ The \texttt{general} latent may fire for many tokens but mysteriously may miss some cases, because those cases were \texttt{absorbed} into other latents.

\paragraph{SAEBench Utilization}
\begin{itemize}
    \item Picking a ground-truth concept (e.g., “starts with S”).
    \item Identifying the SAE latent most responsible for that concept.
    \item Checking where that latent fails to fire.
    \item Seeing if other latents step in to cover the missing cases.
\end{itemize}
The final absorption score is the fraction of a concept’s cases that are not covered by its main latent but are covered by other latents.

\paragraph{Intuitive Example}
Suppose the set of all S-words is:
\begin{align*}
\text{\texttt{short, small, sand, sun, sea, soap, star, slow, sound, sharp}}    
\end{align*}
SAEBench finds that latent $\texttt{L}_\texttt{S}$ is the best candidate for the concept \texttt{starts with S.} Let us assume that $\texttt{L}_\texttt{S}$ fires on 7 out of 10 tokens: \texttt{sand, sun, sea, soap, star, slow, sharp,} and misses the tokens: \texttt{short, small, sound.}
Thus, the recall of the latent $\texttt{L}_\texttt{S} = \frac{7}{10} = 70 \%$.
\\ 

\noindent SAEBench then checks whether other latents cover the missing three tokens. Let us assume the latent $\texttt{L}_\texttt{short}$ fires on tokens \texttt{short} and \texttt{small}, and $\texttt{L}_\texttt{sound}$ fires on the token \texttt{sound}. Now all missing cases are covered, but by different latents.
\\

\noindent Raw absorption = $30 \%$ (because 3/10 cases were absorbed away from the latent $\texttt{L}_\texttt{S}$). Each missing case is explained elsewhere, so they are not pure false negatives. So the score reported would be $30 \%$ absorption, with full compensation of the missing cases by other latents. If none of the missing tokens had been covered elsewhere, the absorption score would flag this as a bad collapse.
\\

\noindent Thus, the absorption score in SAEBench = “How often is a \texttt{general} feature incomplete, and to what extent do other latents pick up the missing pieces?”

\paragraph{Lower absorption scores → better interpretability!}
A low absorption score means the general latent is reliable. It means this latent fires consistently when the concept is present in the input. A tiny portion of any concept is fragmented into other \texttt{specific} latents.

Sparsity($L_0$) defines the number of latents that are allowed to fire per token. Low $L_0$ (high sparsity) implies very few latents fire, high $L_0$ (low sparsity) implies many latents can fire. From the SAEBench results, we can infer that at low $L_0$, SAEs aggressively minimize number of active latents $\implies$ high absorption score.

\subsubsection{Spurious Correlation Coefficient} \label{app:metrics:scr}
SCR score quantifies the capability of an SAE to find and remove spurious correlations (irrelevant signals) while preserving the signal we actually care about.

\paragraph{SAEBench Utilization}
SAEBench takes biased datasets. It uses probes and SAE latents to identify the spurious features (bias-creating features). Then it ablates those latents and checks if classifier accuracy on the actual task improves.

\paragraph{Intuitive Example}
Suppose we have a dataset with biased representation, featuring \texttt{male professors} and \texttt{female nurses}. The task is to predict a person's profession based on their individual characteristics. Here, gender is the spurious feature.\\

\noindent We train a classifier on this biased data, which will learn the \texttt{gender-profession} shortcut. Let us assume:
\begin{align*}
    \text{Accuracy on biased data} (A_{base}) = 95\%
\end{align*}
\noindent Next, we include \texttt{female professors} and \texttt{male nurses} in the dataset, and again train a classifier. Let us assume:
\begin{align*}
    \text{Accuracy on unbiased data} (A_{best}) = 75\%
\end{align*}
Ultimately, we use SAE latents to ablate spurious features. We find \texttt{gender}-related latents, set them to zero, and retrain the classifier on the biased data. Let us assume:
\begin{align*}
    \text{Accuracy on biased data with ablated latents} (A_{abl}) = 80\%
\end{align*}
Finally, we compute the SCR score as follows:
\begin{align*}
    SCR &= \frac{A_{abl} - A_{base}}{A_{best} - A_{base}} \\
    &= \frac{80 - 95}{75 - 95} = 0.75
\end{align*}

\paragraph{Higher SCR scores → better SAE!}
A score of $1.0$ indicates that SAE removed all spurious bias perfectly. Thus, SCR tells you whether an SAE can cut out spurious shortcuts like gender bias while keeping the real task intact.

\subsubsection{Targeted Probe Perturbation} \label{app:metrics:tpp}
TPP score quantifies the capability of an SAE to isolate individual concepts so that removing its latents only affects that concept and leaves others unchanged.

\paragraph{SAEBench Utilization}
SAEBench applies TPP to any multiclass dataset to test whether SAE features for one class are cleanly separated from others. We train probes (small classifiers) for each class. Next, we find the SAE latents most used by a specific probe and ablate them.
\\

\noindent \textit{TPP is a generalization of SCR}. Instead of needing biased datasets, it can test disentanglement on any multiclass dataset.

\paragraph{Intuitive Example}
Suppose we have an Amazon reviews dataset with \texttt{books, electronics, home supplies} and the goal is to check if \texttt{book} latents are isolated.
\\

\noindent We train probes for each class. Let us assume:
\begin{align*}
    \text{Accuracy for Books probe} &= 85\% \\
    \text{Accuracy for Electronics probe} &= 75\% \\
    \text{Accuracy for Home supplies probe} &= 80\%\\
\end{align*}
After ablating book probes, let us assume:
\begin{align*}
    \text{Accuracy for Books probe} &= 55\% \implies \Delta = -30\% \\
    \text{Accuracy for Electronics probe} &= 70\% \implies \Delta = -5\% \\
    \text{Accuracy for Home supplies probe} &= 79\% \implies \Delta = -1\% \\
\end{align*}
Since we are playing around with the target: \texttt{books}, we compute the TPP scores as follows:
\begin{align*}
    TPP &= \text{mean} \Delta_{i=j} - \text{mean} \Delta_{i\neq j} \\
    &= (-30) - (\frac{-5-1}{2})\\
    &= -30 +3 = -27\\ 
\end{align*}
Since TPP is normalized in practice, this would be reported as a high positive score showing the Book latents only hurt Book classification.

\paragraph{Higher TPP scores → better SAE!}
A higher TPP score is better, indicating that the SAE identified features specific to one class. Removing them tanks the target probe but leaves others intact. A lower TPP score means SAE latents are entangled.

\subsubsection{Threshold Sensitivity of SCR and TPP}
\label{sec:threshold}

SCR and TPP are each computed at seven ablation thresholds
$k \in \{2, 5, 10, 20, 50, 100, 500\}$, corresponding to the number of
top-scoring SAE latents ablated per concept direction.
The choice of $k$ is a hyperparameter of the evaluation, not of the model
or the pruning method, and a valid evaluation protocol should not depend
sensitively on it.

We verify this for \texttt{gemma-2-2b} at 50\% sparsity in
Figure~\ref{fig:threshold}, averaging across all layers. At every value of $k$ the ordering $\textsc{SparseGPT} \ge \textsc{Wanda} \ge \textsc{Magnitude}$ holds for both SCR and TPP\@.

\begin{figure}[t]
    \centering
    \includegraphics[width=\textwidth]{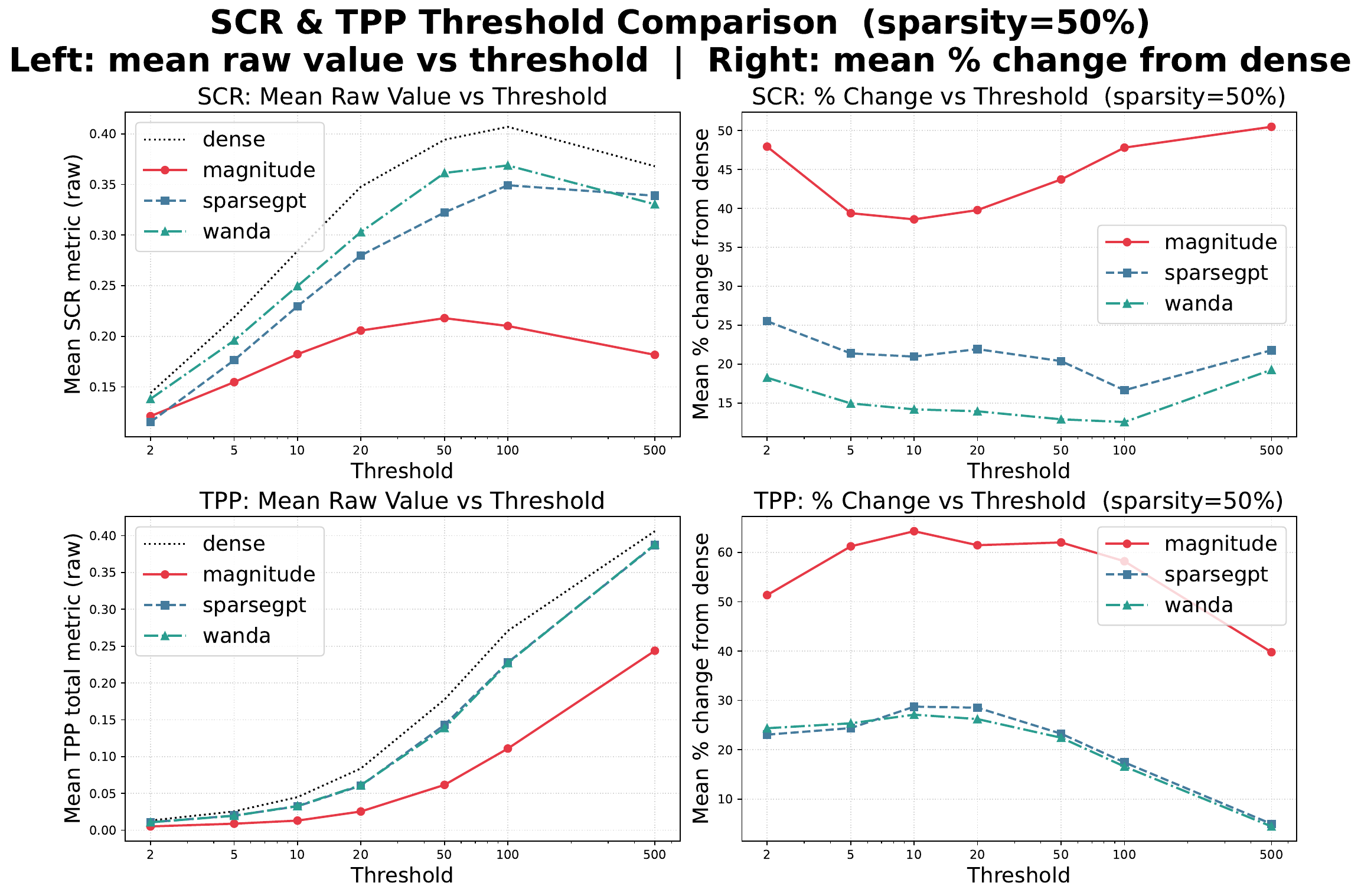}
    \caption{SCR/TPP metric stability across ablation thresholds for
    \texttt{gemma-2-2b} at $50\%$ sparsity.
    \underline{Left column}: mean raw metric values as a function of ablation
    threshold $k \in \{2, 5, 10, 20, 50, 100, 500\}$ for SCR (top)/
    TPP (bottom), alongside the dense baseline (dotted).
    \underline{Right column:} corresponding mean percentage change from the
    dense baseline.}
    \label{fig:threshold}
\end{figure}

Figure~\ref{fig:threshold} examines how SCR and TPP vary with the ablation threshold $k$. The raw metric values for all methods increase with $k$ for both SCR and TPP, as ablating more latents captures more of the relevant concept signal. The method ordering is consistent across all evaluated thresholds: \textsc{magnitude} pruning shows the largest degradation throughout, while \textsc{Wanda} and \textsc{SparseGpt} remain close to each other and substantially closer to the dense baseline. The percentage change from the dense baseline stabilises for $k \geq 10$ for the activation-aware methods, while \textsc{magnitude} pruning shows larger and less stable degradation across thresholds. Based on these observations, we can confirm the choice of $k = 10$ as the canonical threshold, as it lies in the stable regime and produces the clearest separation between pruning paradigms. 

\subsection{Per-Layer Results for \texttt{gemma-2-2b}}
\label{app:perlayer-gemma2b}

Full layer-by-layer SAEBench metrics for \texttt{gemma-2-2b} at $50\%$ sparsity, split by depth group. These results support the aggregated group-level summaries reported in Table~\ref{tab:results_by_group} and the per-layer profiles shown in Figures~\ref{fig:pct_change} and~\ref{fig:heatmap}.
\begin{table}[t]
\centering
\label{tab:results_per_layer_early}
\setlength{\tabcolsep}{3pt}
\small
\begin{tabular}{@{}llr ccccc cc c c@{}}
\toprule
\multirow{2}{*}{\textbf{Layers}} &\multirow{2}{*}{\textbf{Method}} &\multirow{2}{*}{\makecell{\textbf{Spar.}\\\textbf{(\%)}}} &\multicolumn{5}{c}{\textbf{Core}} &\multicolumn{2}{c}{\textbf{Absorption}} &\textbf{SCR} &\textbf{TPP} \\
\cmidrule(lr){4-8}\cmidrule(lr){9-10}\cmidrule(lr){11-11}\cmidrule(lr){12-12}
& & &\makecell{KL\\Score$\uparrow$} &\makecell{CE\\Score$\uparrow$} &\makecell{Expl.\\Var.$\uparrow$} &\makecell{Cos\\Sim$\uparrow$} &\makecell{MSE\\$\downarrow$} &\makecell{Abs.\\Frac.$\downarrow$} &\makecell{Full\\Abs.$\downarrow$} &\makecell{Top-10\\$\uparrow$} &\makecell{Top-10\\$\uparrow$} \\
\midrule
\multirow{4}{*}{\rotatebox[origin=c]{90}{\texttt{Layer 0}}}
 & \textit{Baseline} & \textit{0}
   & 0.9966 & 0.9967 & 0.9453 & 0.9766 & 0.1396
   & 0.0131 & 0.0324
   & 0.1488 & 0.0272 \\
\cmidrule(l){2-12}
 & \spm{Magnitude} & 50
   & 0.9915 & \textbf{0.9962} & 0.8984 & 0.9531 & 0.2754
   & 0.0141 & 0.0432
   & 0.1385 & \textbf{0.0327} \\
 & \spm{Wanda} & 50
   & \textbf{0.9964} & 0.9932 & \textbf{0.9375} & \textbf{0.9727} & \textbf{0.1582}
   & \textbf{0.0127} & \textbf{0.0339}
   & \textbf{0.1605} & 0.0242 \\
 & \spm{SparseGPT} & 50
   & 0.9953 & 0.9948 & \textbf{0.9375} & \textbf{0.9727} & 0.1611
   & 0.0133 & 0.0346
   & 0.1514 & 0.0206 \\
\midrule
\multirow{4}{*}{\rotatebox[origin=c]{90}{\texttt{Layer 1}}}
 & \textit{Baseline} & \textit{0}
   & 0.9958 & 0.9951 & 0.9219 & 0.9648 & 0.2402
   & 0.0989 & 0.0913
   & 0.1788 & 0.0336 \\
\cmidrule(l){2-12}
 & \spm{Magnitude} & 50
   & 0.9914 & \textbf{0.9962} & 0.8750 & 0.9375 & 0.3984
   & 0.1130 & 0.1012
   & 0.1456 & \textbf{0.0357} \\
 & \spm{Wanda} & 50
   & \textbf{0.9956} & 0.9914 & \textbf{0.9141} & \textbf{0.9609} & \textbf{0.2793}
   & 0.0967 & 0.0820
   & \textbf{0.1732} & 0.0241 \\
 & \spm{SparseGPT} & 50
   & 0.9942 & 0.9931 & 0.9102 & 0.9570 & \textbf{0.2793}
   & \textbf{0.0853} & \textbf{0.0807}
   & 0.1602 & 0.0271 \\
\midrule
\multirow{4}{*}{\rotatebox[origin=c]{90}{\texttt{Layer 2}}}
 & \textit{Baseline} & \textit{0}
   & 0.9960 & 0.9967 & 0.9141 & 0.9609 & 0.2578
   & 0.0900 & 0.0224
   & 0.1279 & 0.0184 \\
\cmidrule(l){2-12}
 & \spm{Magnitude} & 50
   & 0.9912 & \textbf{0.9962} & 0.8516 & 0.9219 & 0.4512
   & 0.1002 & 0.0458
   & 0.1024 & \textbf{0.0195} \\
 & \spm{Wanda} & 50
   & \textbf{0.9956} & 0.9914 & \textbf{0.8945} & \textbf{0.9492} & \textbf{0.3008}
   & 0.0880 & 0.0246
   & \textbf{0.1048} & 0.0105 \\
 & \spm{SparseGPT} & 50
   & 0.9943 & 0.9948 & 0.8867 & 0.9453 & 0.3086
   & \textbf{0.0748} & \textbf{0.0226}
   & 0.0870 & 0.0120 \\
\midrule
\multirow{4}{*}{\rotatebox[origin=c]{90}{\texttt{Layer 3}}}
 & \textit{Baseline} & \textit{0}
   & 0.9930 & 0.9934 & 0.8828 & 0.9453 & 0.3906
   & 0.1136 & 0.0570
   & 0.1250 & 0.0190 \\
\cmidrule(l){2-12}
 & \spm{Magnitude} & 50
   & 0.9843 & 0.9772 & 0.8125 & 0.9023 & 0.6211
   & \textbf{0.1027} & 0.0680
   & 0.0971 & 0.0108 \\
 & \spm{Wanda} & 50
   & \textbf{0.9912} & 0.9829 & \textbf{0.8594} & \textbf{0.9336} & \textbf{0.4297}
   & 0.1145 & \textbf{0.0614}
   & \textbf{0.1203} & 0.0142 \\
 & \spm{SparseGPT} & 50
   & 0.9899 & \textbf{0.9880} & 0.8516 & 0.9258 & 0.4512
   & 0.1182 & 0.0704
   & 0.1202 & \textbf{0.0153} \\
\midrule
\multirow{4}{*}{\rotatebox[origin=c]{90}{\texttt{Layer 4}}}
 & \textit{Baseline} & \textit{0}
   & 0.9964 & 0.9967 & 0.9062 & 0.9531 & 0.3320
   & 0.1330 & 0.0387
   & 0.1160 & 0.0278 \\
\cmidrule(l){2-12}
 & \spm{Magnitude} & 50
   & 0.9876 & 0.9848 & 0.8398 & 0.9180 & 0.5352
   & 0.1736 & 0.0736
   & 0.0824 & 0.0144 \\
 & \spm{Wanda} & 50
   & \textbf{0.9957} & 0.9914 & \textbf{0.8828} & \textbf{0.9414} & \textbf{0.3711}
   & \textbf{0.1365} & \textbf{0.0319}
   & \textbf{0.1104} & 0.0154 \\
 & \spm{SparseGPT} & 50
   & 0.9939 & \textbf{0.9931} & 0.8789 & 0.9375 & 0.3848
   & 0.1657 & 0.0552
   & 0.0951 & \textbf{0.0155} \\
\midrule
\multirow{4}{*}{\rotatebox[origin=c]{90}{\texttt{Layer 5}}}
 & \textit{Baseline} & \textit{0}
   & 0.9930 & 0.9918 & 0.8711 & 0.9336 & 0.5312
   & 0.0444 & 0.0212
   & 0.1551 & 0.0159 \\
\cmidrule(l){2-12}
 & \spm{Magnitude} & 50
   & 0.9752 & 0.9582 & 0.8203 & 0.9062 & 0.7812
   & 0.0543 & 0.0316
   & 0.1002 & \textbf{0.0194} \\
 & \spm{Wanda} & 50
   & \textbf{0.9916} & 0.9846 & \textbf{0.8516} & \textbf{0.9258} & \textbf{0.5547}
   & \textbf{0.0492} & \textbf{0.0261}
   & \textbf{0.1303} & 0.0182 \\
 & \spm{SparseGPT} & 50
   & 0.9888 & \textbf{0.9863} & 0.8477 & 0.9219 & 0.5859
   & 0.0570 & 0.0302
   & 0.1232 & 0.0182 \\
\midrule
\multirow{4}{*}{\rotatebox[origin=c]{90}{\texttt{Layer 6}}}
 & \textit{Baseline} & \textit{0}
   & 0.9937 & 0.9934 & 0.8750 & 0.9375 & 0.5352
   & 0.0393 & 0.0134
   & 0.1455 & 0.0334 \\
\cmidrule(l){2-12}
 & \spm{Magnitude} & 50
   & 0.9817 & 0.9772 & 0.8203 & 0.9062 & 0.7656
   & 0.0504 & 0.0252
   & 0.0988 & 0.0226 \\
 & \spm{Wanda} & 50
   & \textbf{0.9923} & 0.9846 & \textbf{0.8594} & \textbf{0.9297} & \textbf{0.5742}
   & \textbf{0.0454} & \textbf{0.0189}
   & \textbf{0.1360} & 0.0245 \\
 & \spm{SparseGPT} & 50
   & 0.9906 & \textbf{0.9897} & 0.8516 & 0.9258 & 0.5977
   & 0.0454 & 0.0217
   & 0.1189 & \textbf{0.0288} \\
\midrule
\multirow{4}{*}{\rotatebox[origin=c]{90}{\texttt{Layer 7}}}
 & \textit{Baseline} & \textit{0}
   & 0.9928 & 0.9918 & 0.8828 & 0.9375 & 0.6602
   & 0.0347 & 0.0164
   & 0.2941 & 0.0324 \\
\cmidrule(l){2-12}
 & \spm{Magnitude} & 50
   & 0.9794 & 0.9734 & 0.8281 & 0.9062 & 0.9688
   & \textbf{0.0471} & 0.0259
   & 0.1569 & 0.0084 \\
 & \spm{Wanda} & 50
   & \textbf{0.9913} & 0.9829 & \textbf{0.8594} & \textbf{0.9258} & \textbf{0.7109}
   & 0.0475 & \textbf{0.0201}
   & \textbf{0.2210} & 0.0173 \\
 & \spm{SparseGPT} & 50
   & 0.9899 & \textbf{0.9880} & 0.8477 & 0.9180 & 0.7656
   & 0.0508 & 0.0222
   & 0.2015 & \textbf{0.0244} \\
\midrule
\multirow{4}{*}{\rotatebox[origin=c]{90}{\texttt{Layer 8}}}
 & \textit{Baseline} & \textit{0}
   & 0.9925 & 0.9918 & 0.8789 & 0.9375 & 0.7539
   & 0.0395 & 0.0255
   & 0.2509 & 0.0354 \\
\cmidrule(l){2-12}
 & \spm{Magnitude} & 50
   & 0.9797 & 0.9772 & 0.8281 & 0.9062 & 1.0859
   & 0.0323 & 0.0231
   & 0.0951 & 0.0049 \\
 & \spm{Wanda} & 50
   & \textbf{0.9914} & 0.9829 & \textbf{0.8555} & \textbf{0.9219} & \textbf{0.8047}
   & \textbf{0.0314} & \textbf{0.0192}
   & \textbf{0.1824} & \textbf{0.0207} \\
 & \spm{SparseGPT} & 50
   & 0.9898 & \textbf{0.9880} & 0.8438 & 0.9141 & 0.8633
   & 0.0352 & 0.0245
   & 0.1383 & 0.0142 \\
\bottomrule
\end{tabular}
\caption{%
    Per-layer SAEBench metrics for \texttt{gemma-2-2b} at $50\%$ sparsity,
    layers~0--8 (Early group).
    The dense baseline is italicised for each layer.
    $\uparrow$ higher is better; $\downarrow$ lower is better.
}
\end{table}

\begin{table}[t]
\centering
\label{tab:results_per_layer_middle}
\setlength{\tabcolsep}{3pt}
\footnotesize
\begin{tabular}{@{}llr ccccc cc c c@{}}
\toprule
\multirow{2}{*}{\textbf{Layers}} &\multirow{2}{*}{\textbf{Method}} &\multirow{2}{*}{\makecell{\textbf{Spar.}\\\textbf{(\%)}}} &\multicolumn{5}{c}{\textbf{Core}} &\multicolumn{2}{c}{\textbf{Absorption}} &\textbf{SCR} &\textbf{TPP} \\
\cmidrule(lr){4-8}\cmidrule(lr){9-10}\cmidrule(lr){11-11}\cmidrule(lr){12-12}
& & &\makecell{KL\\Score$\uparrow$} &\makecell{CE\\Score$\uparrow$} &\makecell{Expl.\\Var.$\uparrow$} &\makecell{Cos\\Sim$\uparrow$} &\makecell{MSE\\$\downarrow$} &\makecell{Abs.\\Frac.$\downarrow$} &\makecell{Full\\Abs.$\downarrow$} &\makecell{Top-10\\$\uparrow$} &\makecell{Top-10\\$\uparrow$} \\
\midrule
\multirow{4}{*}{\rotatebox[origin=c]{90}{\texttt{Layer 9}}}
 & \textit{Baseline} & \textit{0}
   & 0.9917 & 0.9901 & 0.8711 & 0.9297 & 0.9844
   & 0.0555 & 0.0498
   & 0.2386 & 0.0247 \\
\cmidrule(l){2-12}
 & \spm{Magnitude} & 50
   & 0.9664 & 0.9544 & 0.7773 & 0.8789 & 1.3438
   & 0.0555 & 0.0498
   & 0.0858 & 0.0118 \\
 & \spm{Wanda} & 50
   & \textbf{0.9915} & 0.9846 & \textbf{0.8438} & \textbf{0.9141} & \textbf{1.0234}
   & \textbf{0.0429} & \textbf{0.0310}
   & \textbf{0.1830} & \textbf{0.0170} \\
 & \spm{SparseGPT} & 50
   & 0.9888 & \textbf{0.9880} & 0.8281 & 0.9062 & 1.0938
   & 0.0527 & 0.0414
   & 0.1565 & 0.0138 \\
\midrule
\multirow{4}{*}{\rotatebox[origin=c]{90}{\texttt{Layer 10}}}
 & \textit{Baseline} & \textit{0}
   & 0.9909 & 0.9901 & 0.8672 & 0.9258 & 1.1953
   & 0.0708 & 0.0541
   & 0.2430 & 0.0220 \\
\cmidrule(l){2-12}
 & \spm{Magnitude} & 50
   & 0.9733 & 0.9620 & 0.7812 & 0.8789 & 1.6719
   & 0.0967 & 0.0769
   & 0.0921 & 0.0081 \\
 & \spm{Wanda} & 50
   & \textbf{0.9914} & 0.9863 & \textbf{0.8438} & \textbf{0.9141} & \textbf{1.2109}
   & \textbf{0.0701} & \textbf{0.0518}
   & \textbf{0.1949} & 0.0153 \\
 & \spm{SparseGPT} & 50
   & 0.9887 & \textbf{0.9897} & 0.8203 & 0.9062 & 1.3281
   & 0.0834 & 0.0659
   & 0.1593 & \textbf{0.0159} \\
\midrule
\multirow{4}{*}{\rotatebox[origin=c]{90}{\texttt{Layer 11}}}
 & \textit{Baseline} & \textit{0}
   & 0.9900 & 0.9885 & 0.8750 & 0.9258 & 1.4844
   & 0.1061 & 0.0916
   & 0.2557 & 0.0354 \\
\cmidrule(l){2-12}
 & \spm{Magnitude} & 50
   & 0.9649 & 0.9430 & 0.7695 & 0.8711 & 2.2344
   & 0.1256 & 0.1183
   & 0.0720 & 0.0064 \\
 & \spm{Wanda} & 50
   & \textbf{0.9908} & 0.9863 & \textbf{0.8516} & \textbf{0.9102} & \textbf{1.5000}
   & \textbf{0.1020} & \textbf{0.0885}
   & \textbf{0.2178} & \textbf{0.0265} \\
 & \spm{SparseGPT} & 50
   & 0.9883 & \textbf{0.9897} & 0.8203 & 0.8984 & 1.7109
   & 0.1181 & 0.1046
   & 0.1827 & 0.0184 \\
\midrule
\multirow{4}{*}{\rotatebox[origin=c]{90}{\texttt{Layer 12}}}
 & \textit{Baseline} & \textit{0}
   & 0.9896 & 0.9885 & 0.8750 & 0.9219 & 1.5391
   & 0.0992 & 0.0840
   & 0.2924 & 0.0250 \\
\cmidrule(l){2-12}
 & \spm{Magnitude} & 50
   & 0.9736 & 0.9696 & 0.7773 & 0.8750 & 2.0781
   & 0.1212 & 0.0971
   & 0.1018 & 0.0056 \\
 & \spm{Wanda} & 50
   & \textbf{0.9904} & 0.9863 & \textbf{0.8516} & \textbf{0.9102} & \textbf{1.5703}
   & \textbf{0.0987} & \textbf{0.0825}
   & \textbf{0.2181} & \textbf{0.0242} \\
 & \spm{SparseGPT} & 50
   & 0.9879 & \textbf{0.9897} & 0.8164 & 0.8984 & 1.8281
   & 0.1112 & 0.0966
   & 0.2046 & 0.0206 \\
\midrule
\multirow{4}{*}{\rotatebox[origin=c]{90}{\texttt{Layer 13}}}
 & \textit{Baseline} & \textit{0}
   & 0.9885 & 0.9868 & 0.8750 & 0.9297 & 1.9766
   & 0.1344 & 0.1252
   & 0.2864 & 0.0275 \\
\cmidrule(l){2-12}
 & \spm{Magnitude} & 50
   & 0.9692 & 0.9582 & 0.7812 & 0.8828 & 3.0156
   & 0.1609 & 0.1517
   & \textbf{0.2095} & 0.0125 \\
 & \spm{Wanda} & 50
   & \textbf{0.9893} & 0.9846 & \textbf{0.8477} & \textbf{0.9180} & \textbf{2.1250}
   & \textbf{0.1279} & \textbf{0.1183}
   & 0.2008 & \textbf{0.0318} \\
 & \spm{SparseGPT} & 50
   & 0.9871 & \textbf{0.9897} & 0.8242 & 0.9062 & 2.5156
   & 0.1422 & 0.1401
   & 0.1783 & 0.0152 \\
\midrule
\multirow{4}{*}{\rotatebox[origin=c]{90}{\texttt{Layer 14}}}
 & \textit{Baseline} & \textit{0}
   & 0.9885 & 0.9868 & 0.8711 & 0.9297 & 2.1719
   & 0.0333 & 0.0288
   & 0.3756 & 0.0440 \\
\cmidrule(l){2-12}
 & \spm{Magnitude} & 50
   & 0.9562 & 0.9316 & 0.7734 & 0.8750 & 3.3438
   & 0.0399 & 0.0390
   & 0.1078 & 0.0102 \\
 & \spm{Wanda} & 50
   & \textbf{0.9893} & 0.9863 & \textbf{0.8438} & \textbf{0.9180} & \textbf{2.3594}
   & \textbf{0.0354} & \textbf{0.0291}
   & \textbf{0.2755} & \textbf{0.0439} \\
 & \spm{SparseGPT} & 50
   & 0.9869 & \textbf{0.9897} & 0.8281 & 0.9102 & 2.7812
   & 0.0401 & 0.0351
   & 0.2309 & 0.0380 \\
\midrule
\multirow{4}{*}{\rotatebox[origin=c]{90}{\texttt{Layer 15}}}
 & \textit{Baseline} & \textit{0}
   & 0.9872 & 0.9868 & 0.8906 & 0.9375 & 2.3594
   & 0.0314 & 0.0274
   & 0.3937 & 0.0584 \\
\cmidrule(l){2-12}
 & \spm{Magnitude} & 50
   & 0.9700 & 0.9734 & 0.8047 & 0.8945 & 4.0938
   & 0.0382 & 0.0390
   & 0.1860 & 0.0090 \\
 & \spm{Wanda} & 50
   & \textbf{0.9881} & 0.9846 & \textbf{0.8594} & \textbf{0.9258} & \textbf{2.6719}
   & \textbf{0.0308} & \textbf{0.0262}
   & \textbf{0.2847} & \textbf{0.0378} \\
 & \spm{SparseGPT} & 50
   & 0.9857 & \textbf{0.9897} & 0.8398 & 0.9141 & 3.1094
   & 0.0377 & 0.0347
   & 0.2341 & 0.0356 \\
\midrule
\multirow{4}{*}{\rotatebox[origin=c]{90}{\texttt{Layer 16}}}
 & \textit{Baseline} & \textit{0}
   & 0.9862 & 0.9852 & 0.8789 & 0.9336 & 3.2969
   & 0.0318 & 0.0254
   & 0.3583 & 0.0752 \\
\cmidrule(l){2-12}
 & \spm{Magnitude} & 50
   & 0.9566 & 0.9468 & 0.7812 & 0.8789 & 5.7500
   & 0.0334 & 0.0319
   & 0.2037 & 0.0116 \\
 & \spm{Wanda} & 50
   & \textbf{0.9868} & 0.9846 & \textbf{0.8555} & \textbf{0.9219} & \textbf{3.6250}
   & \textbf{0.0283} & \textbf{0.0221}
   & \textbf{0.3144} & 0.0377 \\
 & \spm{SparseGPT} & 50
   & 0.9842 & \textbf{0.9897} & 0.8320 & 0.9102 & 4.2188
   & 0.0317 & 0.0295
   & 0.3069 & \textbf{0.0398} \\
\midrule
\multirow{4}{*}{\rotatebox[origin=c]{90}{\texttt{Layer 17}}}
 & \textit{Baseline} & \textit{0}
   & 0.9859 & 0.9852 & 0.8828 & 0.9375 & 3.5156
   & 0.0381 & 0.0409
   & 0.2987 & 0.0693 \\
\cmidrule(l){2-12}
 & \spm{Magnitude} & 50
   & 0.9496 & 0.9392 & 0.7578 & 0.8672 & 6.1875
   & 0.0434 & 0.0538
   & 0.1120 & 0.0048 \\
 & \spm{Wanda} & 50
   & \textbf{0.9865} & 0.9846 & \textbf{0.8594} & \textbf{0.9258} & \textbf{3.9062}
   & 0.0581 & 0.0580
   & \textbf{0.2699} & \textbf{0.0389} \\
 & \spm{SparseGPT} & 50
   & 0.9839 & \textbf{0.9880} & 0.8398 & 0.9141 & 4.4688
   & \textbf{0.0410} & \textbf{0.0479}
   & 0.2390 & 0.0351 \\
\bottomrule
\end{tabular}
\caption{%
    Per-layer SAEBench metrics for \texttt{gemma-2-2b} at $50\%$ sparsity, layers~9--17 (Middle group).
    The dense baseline is italicised for each layer.
    $\uparrow$ higher is better; $\downarrow$ lower is better.
}
\end{table}

\begin{table*}[t]
\centering
\label{tab:results_per_layer_late}
\setlength{\tabcolsep}{3pt}
\footnotesize
\begin{tabular}{@{}llr ccccc cc c c@{}}
\toprule
\multirow{2}{*}{\textbf{Layers}} &\multirow{2}{*}{\textbf{Method}} &\multirow{2}{*}{\makecell{\textbf{Spar.}\\\textbf{(\%)}}} &\multicolumn{5}{c}{\textbf{Core}} &\multicolumn{2}{c}{\textbf{Absorption}} &\textbf{SCR} &\textbf{TPP} \\
\cmidrule(lr){4-8}\cmidrule(lr){9-10}\cmidrule(lr){11-11}\cmidrule(lr){12-12}
& & &\makecell{KL\\Score$\uparrow$} &\makecell{CE\\Score$\uparrow$} &\makecell{Expl.\\Var.$\uparrow$} &\makecell{Cos\\Sim$\uparrow$} &\makecell{MSE\\$\downarrow$} &\makecell{Abs.\\Frac.$\downarrow$} &\makecell{Full\\Abs.$\downarrow$} &\makecell{Top-10\\$\uparrow$} &\makecell{Top-10\\$\uparrow$} \\
\midrule
\multirow{4}{*}{\rotatebox[origin=c]{90}{\texttt{Layer 18}}}
 & \textit{Baseline} & \textit{0}
   & 0.9855 & 0.9852 & 0.8867 & 0.9375 & 4.4062
   & 0.0951 & 0.1101
   & 0.3244 & 0.0655 \\
\cmidrule(l){2-12}
 & \spm{Magnitude} & 50
   & 0.9426 & 0.9278 & 0.7617 & 0.8672 & 6.9375
   & \textbf{0.0795} & 0.1189
   & 0.1193 & 0.0063 \\
 & \spm{Wanda} & 50
   & \textbf{0.9854} & 0.9846 & \textbf{0.8672} & \textbf{0.9297} & \textbf{4.8750}
   & 0.0813 & \textbf{0.0921}
   & \textbf{0.2761} & 0.0511 \\
 & \spm{SparseGPT} & 50
   & 0.9831 & \textbf{0.9880} & 0.8477 & 0.9180 & 5.5000
   & 0.0815 & 0.1070
   & 0.2658 & \textbf{0.0573} \\
\midrule
\multirow{4}{*}{\rotatebox[origin=c]{90}{\texttt{Layer 19}}}
 & \textit{Baseline} & \textit{0}
   & 0.9844 & 0.9835 & 0.8828 & 0.9375 & 5.5312
   & 0.0823 & 0.0935
   & 0.3398 & 0.0660 \\
\cmidrule(l){2-12}
 & \spm{Magnitude} & 50
   & 0.9600 & 0.9658 & 0.7773 & 0.8789 & 9.3750
   & 0.0787 & 0.1092
   & 0.1583 & 0.0126 \\
 & \spm{Wanda} & 50
   & \textbf{0.9845} & 0.9846 & \textbf{0.8633} & \textbf{0.9297} & \textbf{6.0312}
   & 0.0861 & \textbf{0.0931}
   & \textbf{0.3061} & 0.0447 \\
 & \spm{SparseGPT} & 50
   & 0.9823 & \textbf{0.9880} & 0.8438 & 0.9141 & 6.8750
   & \textbf{0.0767} & 0.0975
   & 0.2888 & \textbf{0.0482} \\
\midrule
\multirow{4}{*}{\rotatebox[origin=c]{90}{\texttt{Layer 20}}}
 & \textit{Baseline} & \textit{0}
   & 0.9827 & 0.9819 & 0.8750 & 0.9336 & 7.2500
   & 0.1144 & 0.1249
   & 0.4253 & 0.0677 \\
\cmidrule(l){2-12}
 & \spm{Magnitude} & 50
   & 0.9571 & 0.9506 & 0.7617 & 0.8711 & 12.5000
   & 0.1091 & 0.1540
   & 0.2972 & 0.0093 \\
 & \spm{Wanda} & 50
   & \textbf{0.9829} & 0.9829 & \textbf{0.8633} & \textbf{0.9297} & \textbf{7.6562}
   & \textbf{0.0949} & \textbf{0.1034}
   & 0.3849 & \textbf{0.0641} \\
 & \spm{SparseGPT} & 50
   & 0.9812 & \textbf{0.9863} & 0.8438 & 0.9180 & 8.8125
   & 0.1039 & 0.1251
   & \textbf{0.3876} & 0.0523 \\
\midrule
\multirow{4}{*}{\rotatebox[origin=c]{90}{\texttt{Layer 21}}}
 & \textit{Baseline} & \textit{0}
   & 0.9798 & 0.9786 & 0.8750 & 0.9336 & 10.0625
   & 0.1283 & 0.1407
   & 0.4160 & 0.0846 \\
\cmidrule(l){2-12}
 & \spm{Magnitude} & 50
   & 0.9551 & 0.9506 & 0.7891 & 0.8867 & 18.7500
   & \textbf{0.1043} & 0.1465
   & 0.4425 & 0.0144 \\
 & \spm{Wanda} & 50
   & \textbf{0.9800} & 0.9795 & \textbf{0.8672} & \textbf{0.9297} & \textbf{10.3750}
   & 0.1150 & \textbf{0.1263}
   & \textbf{0.4694} & 0.0528 \\
 & \spm{SparseGPT} & 50
   & 0.9788 & \textbf{0.9828} & 0.8516 & 0.9219 & 12.1250
   & 0.1182 & 0.1399
   & 0.4662 & \textbf{0.0694} \\
\midrule
\multirow{4}{*}{\rotatebox[origin=c]{90}{\texttt{Layer 22}}}
 & \textit{Baseline} & \textit{0}
   & 0.9777 & 0.9769 & 0.8633 & 0.9297 & 13.9375
   & 0.1217 & 0.1252
   & 0.4038 & 0.0619 \\
\cmidrule(l){2-12}
 & \spm{Magnitude} & 50
   & 0.9520 & 0.9468 & 0.7578 & 0.8750 & 27.2500
   & \textbf{0.0751} & \textbf{0.0979}
   & \textbf{0.4423} & 0.0034 \\
 & \spm{Wanda} & 50
   & \textbf{0.9782} & 0.9777 & \textbf{0.8594} & \textbf{0.9258} & \textbf{14.3750}
   & 0.1108 & 0.1132
   & 0.3853 & 0.0440 \\
 & \spm{SparseGPT} & 50
   & 0.9774 & \textbf{0.9828} & 0.8398 & 0.9180 & 17.0000
   & 0.1170 & 0.1314
   & 0.3808 & \textbf{0.0524} \\
\midrule
\multirow{4}{*}{\rotatebox[origin=c]{90}{\texttt{Layer 23}}}
 & \textit{Baseline} & \textit{0}
   & 0.9752 & 0.9736 & 0.8555 & 0.9258 & 18.3750
   & 0.1120 & 0.1156
   & 0.3858 & 0.0634 \\
\cmidrule(l){2-12}
 & \spm{Magnitude} & 50
   & 0.9541 & 0.9468 & 0.7422 & 0.8672 & 37.5000
   & \textbf{0.0672} & \textbf{0.0836}
   & \textbf{0.4309} & 0.0098 \\
 & \spm{Wanda} & 50
   & \textbf{0.9757} & 0.9743 & \textbf{0.8516} & \textbf{0.9219} & \textbf{18.6250}
   & 0.1053 & 0.1080
   & 0.3862 & \textbf{0.0496} \\
 & \spm{SparseGPT} & 50
   & 0.9749 & \textbf{0.9794} & 0.8320 & 0.9141 & 22.3750
   & 0.1042 & 0.1163
   & 0.3495 & 0.0397 \\
\midrule
\multirow{4}{*}{\rotatebox[origin=c]{90}{\texttt{Layer 24}}}
 & \textit{Baseline} & \textit{0}
   & 0.9695 & 0.9671 & 0.8438 & 0.9219 & 24.8750
   & 0.1410 & 0.1928
   & 0.3933 & 0.0638 \\
\cmidrule(l){2-12}
 & \spm{Magnitude} & 50
   & 0.9528 & 0.9468 & 0.7266 & 0.8633 & 55.0000
   & \textbf{0.0747} & \textbf{0.1230}
   & 0.3507 & 0.0136 \\
 & \spm{Wanda} & 50
   & \textbf{0.9698} & 0.9675 & \textbf{0.8398} & \textbf{0.9219} & \textbf{25.2500}
   & 0.1184 & 0.1606
   & \textbf{0.4123} & 0.0453 \\
 & \spm{SparseGPT} & 50
   & 0.9698 & \textbf{0.9725} & 0.8164 & 0.9102 & 30.8750
   & 0.1087 & 0.1578
   & 0.3605 & \textbf{0.0469} \\
\midrule
\multirow{4}{*}{\rotatebox[origin=c]{90}{\texttt{Layer 25}}}
 & \textit{Baseline} & \textit{0}
   & 0.9662 & 0.9654 & 0.8633 & 0.9258 & 29.8750
   & 0.1717 & 0.1925
   & 0.4197 & 0.0722 \\
\cmidrule(l){2-12}
 & \spm{Magnitude} & 50
   & 0.9545 & \textbf{0.9734} & 0.7188 & 0.8516 & 68.5000
   & \textbf{0.0761} & \textbf{0.1081}
   & 0.3124 & 0.0240 \\
 & \spm{Wanda} & 50
   & \textbf{0.9679} & 0.9675 & \textbf{0.8438} & \textbf{0.9141} & \textbf{30.7500}
   & 0.1239 & 0.1463
   & 0.3742 & 0.0636 \\
 & \spm{SparseGPT} & 50
   & 0.9671 & 0.9725 & 0.8359 & 0.9062 & 37.2500
   & 0.1331 & 0.1591
   & \textbf{0.3850} & \textbf{0.0714} \\
\bottomrule
\end{tabular}
\caption{%
    Per-layer SAEBench metrics for \texttt{gemma-2-2b} at $50\%$ sparsity,
    layers~18--25 (Late group).
    The dense baseline is italicised for each layer.
    $\uparrow$ higher is better; $\downarrow$ lower is better.
}
\end{table*}

\subsection{Empirical Perturbation Energy}
\label{app:pert-energy}

Table~\ref{tab:pert-energy} reports the empirical perturbation energy
$\varepsilon^2 = \mathbb{E}\!\left[\lVert \Delta W\, x_{\mathrm{in}}
\rVert^2\right]$, averaged within each layer group, for each pruning
method on \texttt{gemma-2-2b} at 50\% sparsity, estimated over the
calibration set of Sec.~\ref{sec:setup}. The ordering
$\varepsilon^2(\textsc{SparseGPT}) < \varepsilon^2(\textsc{Wanda}) < 
\varepsilon^2(\textsc{Magnitude})$ predicted by
Corollary~\ref{cor:paradigm} is confirmed at every layer group:
\textsc{Magnitude} incurs roughly twice the perturbation energy of the
activation-aware methods. This moderate $\varepsilon^2$ ratio is amplified into the large SAEBench gaps of Table~\ref{tab:results_by_group} through the quadratic and cross terms of Theorem~\ref{thm:main}, consistent with the bound's structure.

Notably, per-layer $\varepsilon^2$ does not mirror the layer sensitivity
profile of Figure~\ref{fig:layer_sensitivity}: middle layers exhibit the lowest
local $\varepsilon^2$ yet the highest SAE degradation. Middle-layer
vulnerability is therefore driven not by locally larger perturbation but
by the accumulation of upstream perturbation from early layers propagating
through residual connections. This further motivates the layer-wise
allocation of Sec.~\ref{sec:results}: protecting early layers reduces the
perturbation that cascades into vulnerable middle layers.

\begin{table}[t]
\centering
\begin{tabular}{lccc}
\toprule
Layer Group & \textsc{Magnitude} $\varepsilon^2$ & \textsc{Wanda} $\varepsilon^2$ & \textsc{SparseGPT} $\varepsilon^2$ \\
\midrule
Early (0--8)    & 57.36 & 31.87 & 26.66 \\
Middle (9--17)  & 47.13 & 21.16 & 20.70 \\
Late (18--25)   & 53.69 & 26.45 & 24.95 \\
\bottomrule
\end{tabular}
\caption{Empirical perturbation energy $\varepsilon^2$ by layer group for
\texttt{gemma-2-2b} at 50\% sparsity. Lower is better; the ordering
predicted by Corollary~\ref{cor:paradigm} holds in every group.}
\label{tab:pert-energy}
\end{table}
\end{document}